\documentclass{article} %
\usepackage{iclr2027_conference,times}

\usepackage{amsmath,amsfonts,bm,amsthm}

\newtheorem{theorem}{Theorem}[section]

\newtheorem{proposition}{Proposition}[theorem]
\newtheorem{lemma}[theorem]{Lemma}
\newtheorem{definition}{Definition}[section]
\theoremstyle{remark}
\newtheorem{remark}[theorem]{Remark}

\newcommand{\ourmethod}{\textsc{RW-Flow}}
\newcommand{\stopgrad}{\texttt{stopgrad}}

\newcommand{\kmmd}{\textsc{kMMD}}
\newcommand{\mmd}{\textsc{MMD}}
\newcommand{\nna}{\textsc{1-NNA}}
\newcommand{\cov}{\textsc{COV}}

\newcommand{\Mfld}{\mathcal M}
\newcommand{\Exp}[1]{\exp_{#1}}
\newcommand{\Log}[1]{\log_{#1}}
\newcommand{\dg}{d_g}
\newcommand{\volg}{\mathrm{vol}_g}
\newcommand{\divg}{\operatorname{div}_g}
\newcommand{\gradg}{\nabla_{\!g}}
\newcommand{\Tang}[1]{T_{#1}\Mfld}
\newcommand{\Prob}{\mathcal{P}}
\newcommand{\OT}{\mathrm{OT}}
\newcommand{\Sink}{\mathcal{S}}

\newcommand{\dd}{\,\mathrm{d}}
\newcommand{\Meas}{\mathfrak{M}}

\newcommand{\erf}{\mathrm{erf}}
\newcommand{\Real}{\operatorname{Re}}

\definecolor{darkgreen}{RGB}{0,128,0}

\def\figref#1{figure~\ref{#1}}

\def\secref#1{section~\ref{#1}}

\def\eqref#1{equation~\ref{#1}}

\def\algref#1{algorithm~\ref{#1}}

\def\1{\bm{1}}

\DeclareMathAlphabet{\mathsfit}{\encodingdefault}{\sfdefault}{m}{sl}
\SetMathAlphabet{\mathsfit}{bold}{\encodingdefault}{\sfdefault}{bx}{n}

\def\gC{{\mathcal{C}}}
\def\gD{{\mathcal{D}}}

\def\gF{{\mathcal{F}}}

\def\gH{{\mathcal{H}}}

\def\gL{{\mathcal{L}}}

\newcommand{\KL}{D_{\mathrm{KL}}}

\usepackage{hyperref}
\usepackage{url}
\usepackage{microtype}
\usepackage{graphicx}
\usepackage{subcaption}
\usepackage{algorithm}
\usepackage{algpseudocode}
\usepackage{tabularx}
\usepackage{multicol, multirow}
\usepackage{makecell}
\usepackage{verbatim}
\usepackage{xcolor}
\usepackage{amssymb} %
\usepackage{pifont}  %
\usepackage[table]{xcolor}
\usepackage{enumitem}
\usepackage{array}
\usepackage{caption}
\usepackage{adjustbox}
\usepackage{booktabs} %
\usepackage{longtable}
\usepackage[most]{tcolorbox}
\usepackage{wrapfig}
\usepackage{amsmath, thmtools}
\usepackage[normalem]{ulem}
\useunder{\uline}{\ul}{}

\definecolor{metablue}{HTML}{0064E0}

\newcommand{\ie}[0]{\emph{i.e.,}}

\newcommand{\ourname}[0]{\textsc{RW-Flow}}
\renewcommand{\eqref}[1]{Eqn.~\ref{#1}}
\renewcommand{\secref}[1]{Sec.~\ref{#1}}
\renewcommand{\algref}[1]{Alg.~\ref{#1}}
\newcommand{\tabref}[1]{Tab.~\ref{#1}}
\renewcommand{\figref}[1]{Fig.~\ref{#1}}
\newcommand{\appref}[1]{App.~\ref{#1}}
\newcommand{\thmref}[1]{Thm.~\ref{#1}}
\newcommand{\propref}[1]{Prop.~\ref{#1}}
\newcommand{\lemref}[1]{Lem.~\ref{#1}}

\newcommand{\coref}[1]{Cor.~\ref{#1}}

\newcolumntype{Y}{>{\centering\arraybackslash}X}

\title{\ourname{}: One-Step Generation on Compact Manifolds via Wasserstein Gradient Flows}

\iclrfinalcopy

\author{Ualibyek Nurgulan \hspace{0.5em} Seungwoo Yoo\textsuperscript{$\ast$} \hspace{0.5em} Prin Phunyaphibarn\textsuperscript{$\ast$} \hspace{0.5em} Minhyuk Sung  \\
KAIST
\texttt{\{ualibek, dreamy1534, prin10517, mhsung\}@kaist.ac.kr} \\
}

\begin{document}

\maketitle
\begingroup
\renewcommand\thefootnote{}\footnotetext{\textsuperscript{$\ast$}Equal contribution.}
\endgroup
\lhead{Preprint}

\begin{figure}[h]
    \centering
    \includegraphics[width=1\columnwidth]{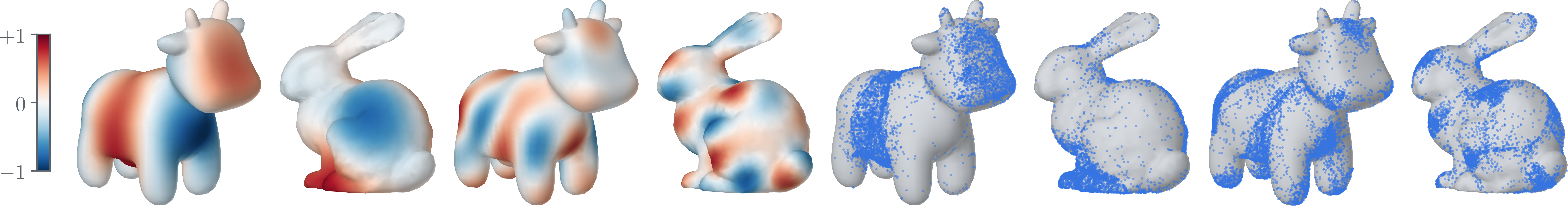}
    \caption{
    \textbf{One-step samples on general manifolds.}
    Left: target eigenfunctions on the Bunny and Spot meshes for
    different eigenfunction indices $k$.
    Right: samples generated by our method in a single function
    evaluation (NFE\,$=$\,1).
    Samples concentrate according to the geometry-induced target
    distributions on the underlying meshes.
    }
    \label{fig:mesh-teaser}
\end{figure}

\begin{abstract}
    Manifold-valued data, and consequently the distributions they induce, are prevalent across many domains, ranging from the locations of geospatial events, such as earthquakes, to biomolecular torsion angles that encode information about three-dimensional structure. While diffusion and flow-based generative models have been successfully extended to compact manifolds, sampling typically requires tens or hundreds of sequential network evaluations. We introduce~\ourmethod{}, a theoretically grounded framework for learning one-step generative models on compact manifolds via Wasserstein gradient flows. The main challenge is~\emph{identifiability}: driving the velocity field to zero should guarantee that the model distribution matches the target distribution. We establish a necessary and sufficient condition for identifiability on compact, connected Riemannian manifolds. We specifically show that, for a symmetric, Lipschitz-continuous cost function, the velocity field induced by the Sinkhorn divergence is identifiable if and only if the associated Gibbs kernel is nondegenerate. This characterization provides a general principle for designing identifiable costs on compact manifolds. It also reveals that the squared geodesic distance, the natural manifold analogue of the squared Euclidean distance, does not always guarantee identifiability. Across benchmarks involving geospatial events, protein side chain torsion angles, RNA backbone torsion angles, and general manifolds discretized as triangular meshes,~\ourmethod{} outperforms existing one-step methods in nearly all settings under fair comparison conditions.
\end{abstract}

\vspace{-1.5\baselineskip}
\section{Introduction}
\vspace{-0.5\baselineskip}

Diffusion and flow-based generative models have been successfully extended from Euclidean data to distributions supported on Riemannian manifolds~\citep{huang2022riemannian,de2022riemannian,chen2024rfm}, enabling applications to manifold-valued data such as geospatial events on spheres and molecular torsion angles on tori. However, sampling remains costly, typically requiring many network evaluations. Recent work has therefore begun adapting few-step acceleration techniques, including flow-map distillation, to Riemannian settings~\citep{davis2026gfm,woo2026riemannian,cheng2026riemannian}. These methods reduce sampling cost but are more difficult to optimize due to redundancy in their parameterizations: they learn maps across multiple pairs of timesteps rather than directly optimizing the one-step map used at inference.

In contrast,~\citet{han2026wflow} recently proposed W-Flow, a framework for learning one-step generative models based on Wasserstein gradient flows (WGFs). W-Flow directly learns a one-step pushforward map by regressing model outputs toward moving targets induced by the Wasserstein gradient of an energy functional that measures the discrepancy between the model and target distributions. A key requirement is \emph{identifiability}: the velocity field should vanish only when the model distribution matches the target distribution. \citet{han2026wflow} considered various divergences as the energy functional and showed that the Sinkhorn divergence~\citep{feydy2019interpolating} with the squared Euclidean distance cost performs best and satisfies the identifiability property in Euclidean spaces.

Given the strong empirical performance of W-Flow~\citep{han2026wflow} in Euclidean domains, a natural question is whether its one-step formulation can be extended to compact manifolds.
A key challenge is that characterizing cost functions that induce identifiable velocity fields is not straightforward, in contrast to the Euclidean setting examined by~\citet{han2026wflow}, where the squared Euclidean distance guarantees identifiability.
This leads to our central question: \emph{which cost functions induce identifiable velocity fields on compact manifolds?}

For compact, connected Riemannian manifolds, we show that there exist many cost functions which induce an identifiable velocity field, opening a new axis for designing one-step models by modifying the cost function. We characterize these cost functions by showing that
for a symmetric, Lipschitz cost function, the induced velocity field is identifiable if and only if the associated Gibbs kernel is nondegenerate. Based on this, we explore broad classes of costs that induce identifiable velocity fields, including spectral costs on general compact manifolds and distance-based costs under suitable geometric conditions. We also analyze the squared geodesic cost, the natural manifold analogue of the squared Euclidean cost, and show that it does not always guarantee identifiability.

On standard benchmarks, including Earth and climate datasets~\citep{noaa2020earthquake,noaa2020volcano,brakenridge2017flood,eosdis2020fire} on spherical domains and protein and RNA torsion-angle datasets~\citep{lovell2003structure,murray2003rna} on tori, we demonstrate that our approach is not only theoretically sound but also empirically effective, surpassing existing methods in nearly all settings under matched model capacity and training time. We further show that spectral costs that are applicable to~\emph{any} compact manifold by construction, enable accelerated sampling on manifolds represented as triangular meshes (\figref{fig:mesh-teaser}). In these settings, our method achieves performance comparable to that of a multi-step baseline while sampling two to three orders of magnitude faster. %

\vspace{-0.5\baselineskip}
\section{Background: W-Flow}\label{sec:background}
\vspace{-0.5\baselineskip}
W-Flow~\citep{han2026wflow} is a powerful framework for learning one-step generative models by evolving the model distribution toward
the target along a Wasserstein gradient flow (WGF) of
an energy functional that measures the discrepancy between the two. Given an energy functional $\mathcal{F}:\mathcal{P}(\mathbb{R}^n)\rightarrow\mathbb{R}$
on the space of probability distributions, its WGF is defined by
\begin{align}
    \label{eq:wgf_pde}
    \partial_t q_t = \nabla \cdot \left(q_t \nabla\frac{\delta\gF}{\delta q}(q_t)\right),
\end{align}
where $\frac{\delta \mathcal{F}}{\delta q}(q_t)$ denotes the first variation of $\mathcal{F}$ at $q_t$. This induces a time-dependent probability path $q_t$ in the Wasserstein space following the steepest descent direction of $\gF$. Equivalently,~\eqref{eq:wgf_pde} can be viewed as a continuity equation with velocity field $V_{q_t,p}=-\nabla\frac{\delta\gF}{\delta q}(q_t)$.

Choosing $\gF(q) = \gD(q, p)$ for a divergence $\gD$ and target distribution $p$ yields the velocity field $V_{q,p} = -\nabla \frac{\delta \gD(\cdot, p)}{\delta q}(q)$, which transports $q_t$ toward $p$.~\citet{han2026wflow} translate this flow into a model-training objective: they transport generated samples along $V_{q_\theta,p}$ and regress the generator outputs toward the transported samples.
\begin{align}
\label{eq:drifting-loss}
    \gL = \mathbb{E}_{z \sim q_0} \left[ {\| f_\theta(z) - \stopgrad(f_\theta(z) + \eta V_{q_\theta,p}(f_\theta(z))) \|^2} \right],
\end{align}
where $f_\theta$ is a one-step transport map, $q_\theta = (f_\theta)_\# q_0$ is the model distribution, and $\eta>0$ is the step size.
Note that~\eqref{eq:drifting-loss} is equivalent to $\eta^2\mathbb{E}_{x\sim q_\theta}\|V_{q_\theta,p}(x)\|^2$: it vanishes
exactly when the velocity field does. Driving the velocity field to zero, however, does not necessarily guarantee that the model converges to the target distribution (\ie $V_{q_\theta, p}\equiv 0$ implies $q_\theta = p$); such a guarantee requires the field to be~\emph{identifiable}.

For identifiability,~\citet{han2026wflow} instantiate $\gD$ with the~\emph{Sinkhorn divergence}~\citep{feydy2019interpolating, ramdas2017wasserstein, genevay2018learning, salimans2018improving},
\begin{align}
\label{eq:sinkhorn}
    \Sink_{\varepsilon,c}(q, p) = \OT_{\varepsilon,c}(q,p) - \tfrac{1}{2}\OT_{\varepsilon,c}(q,q) - \tfrac{1}{2}\OT_{\varepsilon,c}(p,p),
\end{align}
where $\OT_{\varepsilon, c}$ is an entropy-regularized optimal transport cost:
\begin{align}
\label{eq:eot}
    \OT_{\varepsilon,c}(q, p) = \min_{\pi \in \Pi(q,p)} \int c(x,y) \dd\pi(x,y) + \varepsilon \KL(\pi \,\|\, q\otimes p).
\end{align}

Here, $\Pi(q,p)$ is the set of couplings with marginals $q$ and $p$ and $c(x,y)$ specifies the cost of transporting mass from $x$ to $y$.
The minimizer $\pi^{qp}$ is unique and is typically computed in practice by Sinkhorn iterations~\citep{cuturi2013sinkhorn}.
With $c(x,y) = \tfrac{1}{2}\|x-y\|^2$, \citet{han2026wflow} show that the Sinkhorn
divergence is identifiable.

Although W-Flow has achieved state-of-the-art one-step performance on Euclidean domains, including image generation, its extension to compact manifolds remains unexplored. Moreover, while identifiability holds for many cost functions in Euclidean domains, it remains unclear whether any such cost function exists on compact manifolds, motivating the following question:~\emph{Which cost functions guarantee identifiability on compact manifolds, and how can we characterize them?}

\vspace{-0.5\baselineskip}
\section{\ourmethod{}}
\label{sec:method}
\vspace{-0.5\baselineskip}
We introduce \textsc{{\ul R}iemannian {\ul W-Flow}} (\ourname), which extends W-Flow~\citep{han2026wflow} to compact, connected Riemannian manifolds.
In the following, we first derive the WGF
of the Sinkhorn divergence defined with a general cost function on a Riemannian manifold (\secref{subsec:method_riemannian_drift}).
We then provide our main theorem, which characterizes cost functions by connecting identifiability to the nondegeneracy of the associated Gibbs kernels, and present two approaches for verifying this nondegeneracy (\secref{subsec:method_identifiability}). Finally, we explore a variety of cost functions by applying these criteria, including the squared geodesic cost and both spectral and distance-based costs (\secref{subsec:method_costs}).

\vspace{-0.5\baselineskip}
\subsection{Velocity Fields on Compact Manifolds}
\label{subsec:method_riemannian_drift}
\vspace{-0.5\baselineskip}
Let $\Mfld$ denote a compact, connected, smooth manifold equipped with a Riemannian metric $g$, so that $\Mfld$ is complete and the exponential map $\Exp{x}:\Tang{x}\to\Mfld$ is defined on the whole tangent space $\Tang{x}$. We write $\langle\cdot,\cdot\rangle_g$ and $\Vert\cdot\Vert_g$ for the metric and its norm on $\Tang{x}$, $\dg$ for the geodesic distance, $\Log{x}$ for the inverse of $\Exp{x}$ on the domain where it is defined, and $\gradg$, $\divg$ for the Riemannian gradient and divergence.
Given a base distribution $q_0\in\Prob(\Mfld)$ and a target distribution $p\in\Prob(\Mfld)$, our goal is to learn a transport map $f_\theta:\Mfld\to\Mfld$ such that $(f_\theta)_{\#}q_0=p$.

Building on this machinery, we first extend the W-Flow training objective~\citep{han2026wflow} in~\eqref{eq:drifting-loss} to Riemannian manifolds.
In particular,~\eqref{eq:drifting-loss} comprises three components specific to $\mathbb{R}^n$: (1) the additive update $f_\theta(z) + \eta V_{q_\theta,p}(f_\theta(z))$, (2) the squared Euclidean distance between the generator output and the target, and (3) the field $V$ itself.
On a Riemannian manifold $(\Mfld,g)$, the first two components admit natural intrinsic counterparts: a point $x$ is moved along a tangent vector $V(x)\in\Tang{x}$ by the exponential map $\Exp{x}(\eta V(x))$ with $\eta > 0$ denoting a step size, while the discrepancy between two points $x$ and $y$ is measured by the squared geodesic distance $d_g^2(x, y)$. Accordingly, for Riemannian manifolds,~\eqref{eq:drifting-loss} can be rewritten as
\begin{align}
    \gL = \mathbb{E}_{z \sim q_0} \left[ \dg^2\left(f_\theta(z),\ \stopgrad\left(\Exp{f_\theta(z)}(\eta\, V_{q_\theta,p}(f_\theta(z)))\right)\right) \right].
    \label{eq:manifold-loss}
\end{align}

For the exponential map in~\eqref{eq:manifold-loss} to be well defined, the field evaluated at $x=f_\theta(z)$ must lie in the tangent space at $f_\theta (z)$. This condition is readily
satisfied when $V$ is derived from a Wasserstein gradient flow on $\Mfld$. Specifically, the continuity equation in~\eqref{eq:wgf_pde} naturally extends to $\Mfld$ by replacing the Euclidean operators with their Riemannian counterparts~\citep{villani2008optimal}
\begin{equation}
    \label{eq:wgf_pde_manifold}
    \partial_t q_t = \divg \left(q_t \gradg\frac{\delta\gF}{\delta q}(q_t)\right).
\end{equation}

Analogous to W-Flow~\citep{han2026wflow}, we choose the debiased Sinkhorn divergence as our energy functional: $\mathcal{F}(q)=\Sink_{\varepsilon,c}(q,p)$. Here, $\OT_{\varepsilon,c}$ and $\Sink_{\varepsilon,c}$ are defined as in~\eqref{eq:eot} and~\eqref{eq:sinkhorn}, respectively.
Then, the resulting velocity field, obtained as the Wasserstein gradient of $\Sink_{\varepsilon,c}$, is given below.
\begin{restatable}{theorem}{thmmain}
\label{thm:main}
Let $\gF(q)=\Sink_{\varepsilon, c}(q,p)$ on a compact
$(\Mfld,g)$ and $c$ be a differentiable in its
first argument. Then the Wasserstein gradient flow velocity
$V=-\gradg\tfrac{\delta \gF}{\delta q}$ is
\begin{equation}
    V^{\varepsilon,c}_{q, p}(x)\;=\;
    -\int_\Mfld \gradg^{(1)}c(x,y)\;\pi^{qp}(\dd y\mid x)
    \;+\;
    \int_\Mfld \gradg^{(1)}c(x,z)\;\pi^{q q}(\dd z\mid x)
    \;\in\;\Tang{x},
    \label{eq:main}
\end{equation}
where $\pi^{qp}(\dd y\mid x)$ denotes the disintegration with respect to $x$ of the optimal coupling $\pi^{qp}$ of the entropy-regularized $\OT_{\varepsilon,c}(q, p)$ in~\eqref{eq:eot}, and $\gradg^{(1)}$ is the Riemannian gradient in the first argument.
\end{restatable}
The derivation can be found in~\appref{app:proof-main}.
Intuitively, the first term shifts $x$ toward the data points to which it is coupled, whereas the second term induces repulsion among the generated samples and cancels the entropic bias, ensuring that the field vanishes when $q=p$.

An immediate consequence of~\thmref{thm:main} is that the cost function determines the form of the velocity field and, consequently, its identifiability. In the next subsection, we present a theoretical criterion for determining whether a given cost function induces an identifiable velocity field.

\vspace{-0.5\baselineskip}
\subsection{Identifiability via Nondegeneracy}
\label{subsec:method_identifiability}
\vspace{-0.5\baselineskip}
Recall from~\secref{sec:background} that a velocity field $V_{q_\theta,p}$ is~\emph{identifiable} if $V_{q_\theta,p}\equiv 0$ implies $q_\theta=p$. Since the training objective in~\eqref{eq:manifold-loss} drives the norm of $V_{q_\theta,p}$ to zero,
identifiability ensures that the generator does not reach a spurious fixed point even when $q_\theta\neq p$.
While the squared Euclidean cost is sufficient to ensure identifiability in Euclidean spaces~\citep{han2026wflow}, we extend this result by characterizing a broader class of cost functions that preserve identifiability.

Our main result, presented below, shows that for a symmetric, Lipschitz continuous cost function $c$, identifiability of the velocity field induced by the Sinkhorn divergence is determined by the~\emph{nondegeneracy} of the corresponding Gibbs kernel $k=e^{-c/\varepsilon}$, defined below.

\begin{definition}
    Let $(k*\sigma)(x):=\int k(x,y)\,\dd\sigma(y)$ be an integral operator for $\sigma\in\Meas_b(\Mfld).$
    We call $k$ nondegenerate
    if $\sigma\mapsto k*\sigma$ is injective on
    finite signed measures $\Meas_b(\Mfld)$. Otherwise, it is degenerate.
\end{definition}

With this notion in place, we state our main theorem, which provides a criterion for determining which cost functions induce identifiable velocity fields and which do not.

\begin{restatable}{theorem}{thmidentifiable}
\label{thm:identifiable}
Let $(\Mfld, g)$ be a compact, connected Riemannian manifold
with a symmetric, Lipschitz cost function
$c: \Mfld\times\Mfld\to \mathbb{R}$. Denote its Gibbs
kernel $k:=e^{-c/\varepsilon}.$
Then $V^{\varepsilon,c}_{q, p}$ is identifiable if and only if $k$ is nondegenerate.
\end{restatable}

\thmref{thm:identifiable} implies a broader class of admissible cost functions. Specifically, if a cost function is 
symmetric and Lipschitz continuous and its associated Gibbs kernel is nondegenerate, then
the corresponding velocity field is identifiable, making~\eqref{eq:manifold-loss} suitable for learning the generator.
Motivated by this result, we next explore a broad range of cost functions in~\secref{subsec:method_costs}. Here, as a preliminary step, we introduce two approaches to verify the nondegeneracy of their associated Gibbs kernels:~\emph{spectral analysis} and~\emph{universality}.

\vspace{-0.5\baselineskip}
\paragraph{Verifying Nondegeneracy via Spectral Analysis.}
For kernels that are diagonalizable in the Laplace-Beltrami eigenbasis, 
their nondegeneracy can be determined based on their spectra.

\begin{restatable}{lemma}{lemspectral}
\label{lem:spectral}
    Let $\Mfld$ be compact and $k$ a continuous symmetric kernel on $\Mfld$.
    Let $\{\phi_n\}_{n\ge 0}$ be
    Laplace-Beltrami eigenfunctions such that
    they form a basis for $L^2(\Mfld)$ and
    \begin{equation}
        \int k(x,y)\,\phi_n(y)\,\dd\volg(y) = \hat k_n\,\phi_n(x)
        \qquad \text{for all } n\ge 0.
    \end{equation}
    Then $k$ is nondegenerate if and only if $\hat k_n\neq 0$ for all $n\ge 0$.
\end{restatable}

The proof is deferred to \appref{app:proofs}. Notably, the condition in~\lemref{lem:spectral} applies to every spectral kernel on any compact manifold (\secref{sec:spectral-cost}), as well as to every kernel that depends only on the geodesic distance on $\mathbb{S}^d$ or $\mathbb{T}^d$. Importantly,~\lemref{lem:spectral} does not require $k$ to be positive definite: the coefficients $\hat k_n$ may be negative, provided that none of them vanishes.
\vspace{-0.5\baselineskip}
\paragraph{Verifying Nondegeneracy via Universality.}
When the spectrum is not available in closed form but the kernel is positive definite, nondegeneracy can instead be verified through \emph{universality}
(refer to \appref{app:other-energy-functionals} for definition), following~\citet{JMLR:v7:micchelli06a,JMLR:v12:sriperumbudur11a}.
\begin{proposition}[\cite{JMLR:v7:micchelli06a}]\label{prop:uni_nondeg}
    Let $\Mfld$ be compact and $k$ a continuous positive definite kernel
    on $\Mfld$. Then $k$ is nondegenerate if and only if it is universal.
\end{proposition}

Hence, when the Gibbs kernel associated with a cost function is positive definite, it suffices to verify its universality to establish nondegeneracy. Together with~\thmref{thm:identifiable}, the spectral and universality approaches provide complementary criteria for determining whether a given cost function induces an identifiable velocity field. With these results, we explore various cost functions in the following.

\vspace{-0.5\baselineskip}
\subsection{Cost Functions on Compact Manifolds}
\label{subsec:method_costs}
\vspace{-0.5\baselineskip}
In the Euclidean case, the squared Euclidean cost guarantees identifiability~\citep{han2026wflow}. It is therefore natural to ask whether its manifold analogue, the squared geodesic cost $c(x,y)=\frac{1}{2}d_g^2(x,y)$, also induces an identifiable velocity field. This cost satisfies the assumptions of~\thmref{thm:identifiable}; consequently, it suffices to determine whether the associated Gibbs kernel
$\exp\bigl(-\frac{1}{2\varepsilon}d_g^2\bigr)$
is nondegenerate.

\begin{restatable}{proposition}{propgaussnondeg}
    \label{prop:gaussian_kernels_almost_nondegenerate}
    For the geodesic Gaussian kernel $k_\varepsilon = \exp\big(-\tfrac{1}{2\varepsilon} d_g^2\big)$ on $\mathbb{S}^d$ or $\mathbb{T}^d$, the set
    \begin{equation}
        \Gamma = \{\varepsilon > 0 \mid k_\varepsilon \text{ is degenerate}\}
    \end{equation}
    is at most countable. Particularly, for $\mathbb{T}^d$, $\Gamma$ is countably infinite.
\end{restatable}
The proof is deferred to~\appref{app:proofs}. Since Gaussian kernels built on the geodesic distance are
not positive definite for any $\varepsilon > 0$
on many manifolds \citep{feragen2015geodesic, da2023gaussian, da2025invariant, li2024gaussian},
\propref{prop:uni_nondeg} does not apply directly; we instead rely
on~\lemref{lem:spectral}. On both $\mathbb{S}^d$ and $\mathbb{T}^d$, any kernel that depends only on the
geodesic distance is diagonalizable as in~\lemref{lem:spectral}, and the spectral coefficient associated with each mode vanishes for at most countably many values of $\varepsilon$. Therefore, on $\mathbb{S}^d$ and $\mathbb{T}^d$, the squared geodesic cost induces an identifiable velocity field for almost every $\varepsilon>0$. As we show in \secref{sec:geodesic-cost}, dropping the square resolves this on $\mathbb{S}^d$.

However, our argument does not extend to general compact manifolds, motivating us to seek cost functions whose Gibbs kernels are provably nondegenerate for every $\varepsilon > 0$. Through the lens of~\thmref{thm:identifiable}, we explore two such classes. The first, which we call~\emph{spectral costs}, is defined using the Laplace-Beltrami eigenfunctions (\secref{sec:spectral-cost}). The second, which we call~\emph{distance-based costs}, includes the chordal and geodesic costs whose guarantees hold under conditions on the manifold (\secref{sec:nonspectral-cost}). In the following, we introduce these cost functions and prove the identifiability of their induced velocity fields using~\thmref{thm:identifiable}, together with the spectral and universality criteria developed previously. We organize the cost families along with their guarantees and formulas in \tabref{tab:cost-summary}.

\vspace{-0.5\baselineskip}
\subsubsection[Spectral Costs]{Spectral Costs}
\label{sec:spectral-cost}
\vspace{-0.5\baselineskip}
A \emph{spectral cost} is a cost function induced by a spectral kernel $k_\rho: \Mfld \times \Mfld \rightarrow \mathbb{R}$, defined on any compact Riemannian manifold $\Mfld$. We begin by defining spectral kernels below.

\begin{definition}[Spectral Kernels]
    \label{def:spectral_kernel}
    Let $\{(\lambda_n, \phi_n)\}_{n=0}^\infty$ denote the eigenpairs of
    the Laplace-Beltrami operator on a compact Riemannian manifold
    $(\Mfld, g)$, with $\{\phi_n\}$ orthonormal in $L^2(\Mfld)$.
    Spectral kernels are defined by
    \begin{equation}
        \label{eq:spectral_kernel}
        k_\rho(x, y) = \sum_{n=0}^\infty \rho(\lambda_n)\phi_n(x)\phi_n(y),
    \end{equation}
    where the spectral density $\rho\colon [0,\infty)\to[0,\infty)$
    decays fast enough that the series converges uniformly.
\end{definition}

A key property of spectral kernels is that their nondegeneracy is characterized by their spectral density $\rho$, as stated in the following proposition.

\begin{restatable}[Nondegeneracy of Spectral Kernels]{proposition}{propspec}
    \label{prop:spec_posuniversal}
    A spectral kernel $k_\rho$ is nondegenerate if and only if
    $\rho(\lambda_n) \neq 0$ for all $\lambda_n$.
\end{restatable}

\begin{proof}
    Let $\{{\lambda_n, \phi_n}\}_{n\ge 0}$ be the eigenpairs of the Laplace-Beltrami operator
    on $\Mfld.$ Consider
    \begin{equation}
        k * (\phi_n \volg) = \int k_\rho (x, y ) \phi_n(y) \dd\volg(y) = \rho(\lambda_n)\phi_n(x)
    \end{equation}
    Now by \lemref{lem:spectral}, $k_\rho$ is nondegenerate if and only if
    $\rho(\lambda_n)\neq 0$ for all $\lambda_n.$
\end{proof}
Using this property, we can show that nondegenerate spectral kernels induce cost functions that,
when used to define the Sinkhorn divergence in~\eqref{eq:eot}, yield identifiable velocity fields, as stated in the following result.

\begin{restatable}{corollary}{corspecident}
    \label{cor:spec_ident}
    Let $k_\rho$ be a spectral kernel defined on a compact,
    connected manifold $\Mfld$ with a nonvanishing
    spectral density $\rho$ with $k_\rho > 0$ and $\log k_\rho$ is Lipschitz. Then,
    the velocity $V$ in~\eqref{eq:main} with the cost
    $c_{k_\rho}=-\varepsilon\log k_\rho$ is identifiable.
\end{restatable}

We identify several kernels that satisfy the condition in~\coref{cor:spec_ident} and summarize them in~\tabref{tab:spectral-densities}. Accordingly, the Riemannian gradient $\gradg^{(1)}c_{k_\rho}$ is obtained by differentiating the truncated sum. Furthermore, the eigenfunctions are spherical harmonics on the sphere and Fourier modes on flat tori, both of which yield closed-form expressions.

\begin{table}[t!]
\caption[Summary of cost functions]{\textbf{Overview of cost functions considered.} $P_x$ denotes the orthogonal projection onto $T_x\Mfld$ for $\Mfld\subset\mathbb{R}^D$. The last column states for which $\varepsilon$ and on which manifolds the Gibbs kernel is nondegenerate, so that the velocity field is identifiable by \thmref{thm:identifiable}; ``a.e.'' means all but countably many. On manifolds not listed, no guarantee is established.}
\centering
\footnotesize
\setlength{\tabcolsep}{3pt}
\begin{tabularx}{\linewidth}{c|lll>{\raggedright\arraybackslash}X}
\toprule
Cost & $c(x,y)$ & Gibbs kernel $e^{-c/\varepsilon}$ & $\gradg^{(1)}c(x,y)$ & Identifiable for \\
\midrule
Squared Geodesic & $\tfrac12\dg(x,y)^2$ & $e^{-\dg(x,y)^2/2\varepsilon}$ & $-\Log{x}(y)$ & a.e.\ $\varepsilon$ on $\mathbb{S}^d$, $\mathbb{T}^d$  \\
Spectral & $-\varepsilon\log k_\rho(x,y)$ & $\sum_n \rho(\lambda_n)\,\phi_n(x)\phi_n(y)$ & $-\varepsilon\,\gradg^{(1)}\!\log k_\rho(x,y)$ & all $\varepsilon$, any $\Mfld$ \\
Chordal & $\|x-y\|^2$ & $e^{-\|x-y\|^2/\varepsilon}$ & $2P_x(x-y)$ & all $\varepsilon$, any $\Mfld\subset\mathbb{R}^D$  \\
Geodesic & $\dg(x,y)$ & $e^{-\dg(x,y)/\varepsilon}$ & $-\Log{x}(y)/\dg(x,y)$ & all $\varepsilon$ on $\mathbb{S}^d$  \\
\bottomrule
\end{tabularx}
\label{tab:cost-summary}
\vspace{-\baselineskip}
\end{table}

\begin{wraptable}{r}{0.52\textwidth}
    \vspace{-\baselineskip}
    \caption[Spectral densities of common kernels]{%
    \textbf{Spectral densities of nondegenerate spectral kernels.}
    When used as costs in the Sinkhorn divergence~\eqref{eq:eot},
    these kernels yield identifiable velocity fields.
    See \appref{app:spectral-kernels} for details.}
    \centering
    \footnotesize
    \label{tab:spectral-densities}
    \renewcommand{\arraystretch}{1.35}
    \begin{tabular}{@{}l c l@{}}
        \toprule
        Kernel & $\rho(\lambda_n)$ & Parameters \\
        \midrule
        Mat\'ern
            & $\sigma^2\bigl(\tfrac{2\nu}{\kappa^2}+\lambda_n\bigr)^{-\nu-d/2}$
            & $\nu>\tfrac12,\ \kappa,\sigma^2>0$ \\
        Heat
            & $e^{-t\lambda_n}$
            & $t>0$ \\
        Sub.\ Heat
            & $e^{-t\lambda_n^{\alpha}}$
            & $t>0,\ \alpha\in(0,1]$ \\
        \bottomrule
    \end{tabular}
    \vspace{-1.5\baselineskip}
\end{wraptable}

\vspace{-0.5\baselineskip}
\subsubsection{Distance-Based Costs}
\label{sec:nonspectral-cost}
\vspace{-0.5\baselineskip}
Distance-based costs also fit within our framework: under suitable conditions on the manifold, they too guarantee identifiability. We discuss two such examples: the chordal cost and the geodesic cost.

\vspace{-0.5\baselineskip}
\paragraph{Chordal Cost.}
\label{sec:chordal-cost}
When $\Mfld$ is embedded in $\mathbb{R}^D$, the chordal kernels $\exp{(-\lambda\|x-y\|^2)}$~\citep{steinwart2001influence}
are known to be positive definite and universal for any compact subset of $\mathbb{R}^D$. In particular, they are
nondegenerate by \propref{prop:uni_nondeg}.
Using the relation between the cost and its
Gibbs kernel in~\thmref{thm:identifiable}, we identify the induced cost as the squared chordal
distance $\|x-y\|^2$ restricted to $\Mfld$, which is Lipschitz on a compact manifold and symmetric, making it identifiable.

\vspace{-0.5\baselineskip}
\paragraph{Geodesic Cost.}
\label{sec:geodesic-cost}
The Laplacian kernel $\exp(-\lambda d_g(x,y))$ is universal on the sphere ($\mathbb{S}^d$) for every $\lambda>0$~\citep{qidimension2024}.
By \propref{prop:uni_nondeg}, it is nondegenerate.
With $\lambda=1/\varepsilon$, it is the Gibbs kernel of the geodesic distance itself,
\begin{equation}
c(x,y) = d_g(x,y),
\label{eq:geodesic-cost}
\end{equation}
which is a symmetric, Lipschitz cost on $\mathbb{S}^d$. Therefore, identifiability follows from~\thmref{thm:identifiable},
in contrast to the squared geodesic cost.

\vspace{-0.5\baselineskip}
\section{Related Work}

\vspace{-0.5\baselineskip}
\paragraph{Generative Modeling on Manifolds.}
Early work extended normalizing flows to non-Euclidean domains, including hyperbolic spaces, spheres, and tori \citep{bose2020latent,rezende2020normalizing,mathieu2020riemannian}. More recent work adapted diffusion and flow-based models to Riemannian manifolds~\citep{huang2022riemannian,de2022riemannian,chen2024rfm}. However, as with their Euclidean counterparts, sampling from these models remains computationally expensive because it requires multiple sequential steps to integrate an SDE or ODE.

\vspace{-0.5\baselineskip}
\paragraph{Few-Step Generative Models.}
To reduce the sampling cost of diffusion and flow models, several methods distill pretrained multi-step models into few-step generators~\citep{salimans2022progressive,yin2024dmd,zhou2024sid}. However, these approaches require an additional distillation stage on top of teacher pretraining. Other methods instead train few-step generators from scratch by incorporating the underlying SDE or ODE dynamics directly into the training objective~\citep{song2024icm,boffi2026flowmap,geng2026meanflow}. Drifting models~\citep{deng2026driftingmodel} follow this paradigm to learn one-step generators from scratch. During training, they evolve the model distribution along a velocity field that characterizes its discrepancy from the target distribution. W-Flow~\citep{han2026wflow} further introduces an identifiable velocity field, improving empirical performance while providing stronger guarantees that the learned distribution converges to the target.

\vspace{-0.5\baselineskip}
\paragraph{Few-step Generative Modeling on Manifolds.}
Several recent works accelerate generation on manifolds by adapting few-step techniques originally developed for Euclidean settings.~\citet{woo2026riemannian} and~\citet{davis2026gfm} learn shortcut mappings that avoid repeatedly integrating an ODE or SDE, while~\citet{cheng2026riemannian} extend consistency models~\citep{song2023consistency} to manifolds.
Concurrently,~\citet{esteban2026kernel} generalize drifting models~\citep{deng2026driftingmodel} by representing the drifting field as a weighted average of kernel gradients, extending them to arbitrary manifolds. They show that this field corresponds to the WGF of the KL divergence between smoothed probability densities. In contrast, our work adopts the Sinkhorn divergence as the primary energy functional and additionally investigates the squared Maximum Mean Discrepancy (MMD) in~\appref{app:other-energy-functionals}.

\vspace{-0.5\baselineskip}
\section{Experiments}
\label{sec:results}
\vspace{-0.5\baselineskip}

\subsection{Experiment Setup}
\label{sec:experiment_setup}

\vspace{-0.5\baselineskip}
\paragraph{Benchmarks.} To demonstrate the effectiveness of our method, we conduct extensive experiments on standard benchmarks spanning several manifolds. Specifically, we evaluate~\ourmethod{} on the task of modeling distributions over a sphere ($\mathbb{S}^2$), using geospatial event data over the Earth~\citep{mathieu2020riemannian}. We also consider distributions over torus geometry, with representative examples being protein side chain ($\mathbb{T}^2$)~\citep{lovell2003structure}, and RNA backbone torsion angles ($\mathbb{T}^7$)~\citep{murray2003rna}. Furthermore, we consider modeling distributions on general manifolds using mesh-based examples and distributions from~\citet{chen2024rfm}.

\vspace{-\baselineskip}
\paragraph{Evaluation Metrics.} Across all benchmarks, we assess sample quality using the following metrics:~\kmmd{} (Maximum Mean Discrepancy),~\mmd{} (Minimum Matching Distance),~\nna{} (1-Nearest-Neighbor Accuracy), and~\cov{} (Coverage). Each metric is computed between the test data and the generated samples. 

\vspace{-\baselineskip}
\paragraph{Baselines.} We compare our method against state-of-the-art one- and few-step generative models designed for distributions on manifolds. Specifically, we consider Riemannian Flow Matching (RFM)~\citep{chen2024rfm}, Generalised Flow Map (GFM)~\citep{davis2026gfm}, Riemannian MeanFlow (RMF)~\citep{woo2026riemannian}, and Riemannian Consistency Model (RCM)~\citep{cheng2026riemannian}. We additionally include the concurrent work Kernel-Gradient Drifting Models (KGD)~\citep{esteban2026kernel} as a baseline. To ensure a fair comparison, we match the number of model parameters, training wall-clock time, and optimizer configurations across all methods. We use the official implementations of each baseline, making only minor modifications to integrate them with our standalone evaluation script. Full experimental details are provided in~\appref{app:experiment-details}.

\newcommand{\ph}{0.xxx}
\newcommand{\phdset}{\ph & \ph & \ph & \ph}                             %
\newcommand{\phrow}{\phdset & \phdset & \phdset & \phdset}               %
\newcommand{\phpair}{\ph & \ph}                                         %
\newcommand{\phtorus}{\phpair & \phpair & \phpair & \phpair & \phpair}   %

\vspace{-0.5\baselineskip}
\subsection{Sphere -- $\mathbb{S}^2$}
\vspace{-0.5\baselineskip}
We evaluate~\ourmethod{} on the task of modeling distributions on spheres. 
Following prior work~\citep{de2022riemannian,chen2024rfm,davis2026gfm}, we use the geospatial datasets introduced by~\citet{mathieu2020riemannian}, collected from a variety of sources, including~\citet{noaa2020earthquake, noaa2020volcano},~\citet{brakenridge2017flood}, and~\citet{eosdis2020fire}. The main quantitative results are reported in~\tabref{tab:sphere-results}. As shown, our method outperforms current state-of-the-art methods across nearly all benchmarks in the one-step setting, with the geodesic and subordinated heat cost variants achieving the strongest overall performance. Notably, both costs induce identifiable velocity fields on the sphere.

Alongside the quantitative evaluation,~\figref{fig:sphere_qualitative} presents a qualitative comparison between~\ourmethod{} and  GFM~\citep{davis2026gfm}, the strongest baseline. The visual results reflect the gap observed in the quantitative metrics: on benchmarks such as `Wildfire' and `Flood', GFM produces overly diffuse samples, whereas~\ourmethod{} closely captures the reference distributions.

\begin{figure}
    \centering
    \includegraphics[width=1\linewidth]{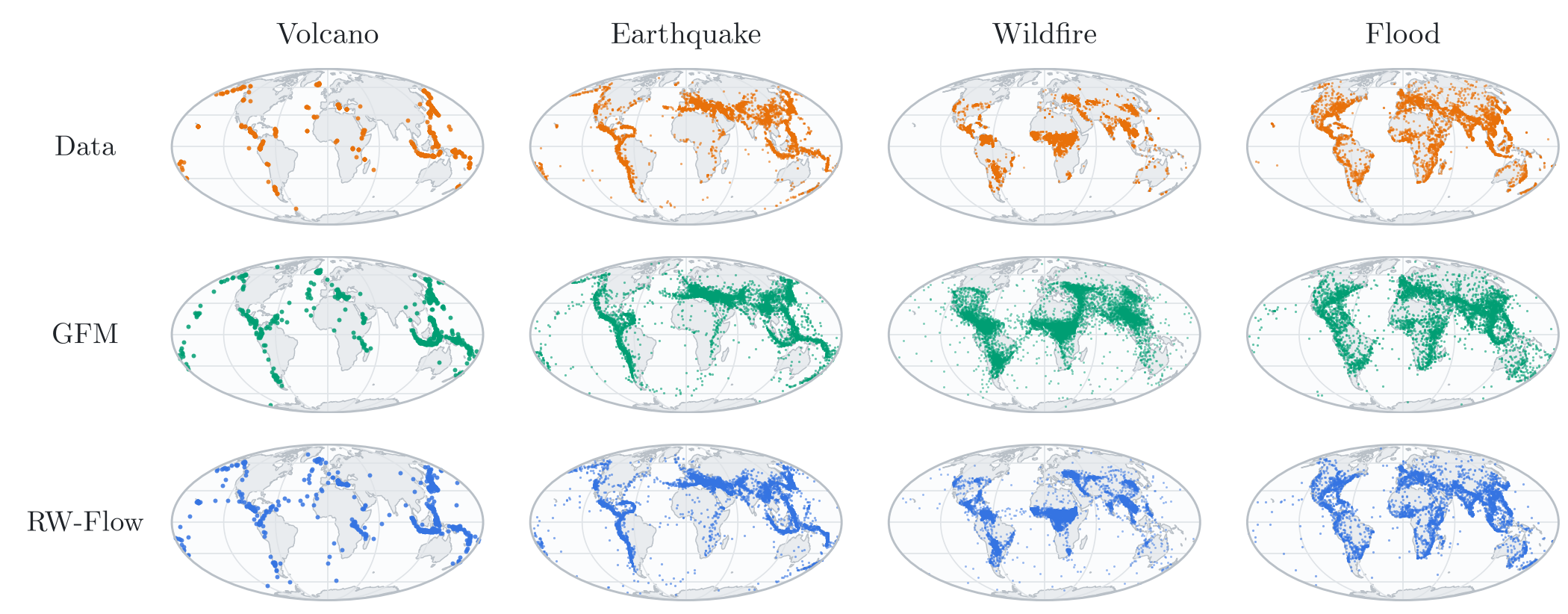}
    \caption{\textbf{Generated samples on geospatial event data benchmarks at 1 NFE.} Each column visualizes the dataset (top) and samples generated by GFM~\citep{davis2026gfm} (middle) and~\ourmethod{}-S (bottom).~\ourmethod{}-S more accurately captures the underlying data distributions, whereas GFM produces diffuse samples that deviate substantially from the ground-truth distributions.}
    \label{fig:sphere_qualitative}
    \vspace{-\baselineskip}
\end{figure}

\newcommand{\rothdr}[1]{\rotatebox{90}{#1}}
\newcommand{\hdrsphere}{\rothdr{\kmmd$\downarrow$} & \rothdr{\mmd$\downarrow$} & \rothdr{\cov$\uparrow$} & \rothdr{\nna}}
\providecommand{\hdrpaper}{\rothdr{\kmmd$\downarrow$} & \rothdr{\cov$\uparrow$} & \rothdr{\nna}}
\begin{table*}[h]
\scriptsize
\setlength{\tabcolsep}{2pt}
\caption[Results on geospatial event data benchmarks]{\textbf{Results on geospatial event data benchmarks.} \kmmd{}, \cov{}, and \nna{} (closest to 0.5 is best) on the volcano, earthquake, flood, and fire datasets (NFE\,$=$\,1; \textbf{bold} = best, \underline{underline} = second best; the held-out row is the val-vs-test reference and is excluded from ranking.) Suffixes denote variants: CT (Consistency Training)
sCT (simplified Consistency Training) L (Lagrangian), E (Eulerian), S (semigroup), MF (mean flow); GG (geodesic Gaussian), GL (geodesic Laplacian), M (Mat\'ern); $v$ denotes
$v$-prediction; \ourmethod{} costs: SG (squared geodesic), C (chordal), G (geodesic), H (heat), M (Mat\'ern), S (subordinated). The metrics are averaged over 3 runs. Full results with standard deviations are in~\tabref{tab:sphere-results-full}.}
\begin{tabularx}{\linewidth}{l*{12}{Y}}
\toprule
 & \multicolumn{3}{c}{Volcano} & \multicolumn{3}{c}{Earthquake} & \multicolumn{3}{c}{Flood} & \multicolumn{3}{c}{Fire} \\
\cmidrule(lr){2-4} \cmidrule(lr){5-7} \cmidrule(lr){8-10} \cmidrule(lr){11-13}
Method
 & \hdrpaper
 & \hdrpaper
 & \hdrpaper
 & \hdrpaper \\
\midrule
\textit{Held-out} & $0.153$ & $0.451$ & $0.555$ & $0.026$ & $0.570$ & $0.471$ & $0.025$ & $0.548$ & $0.509$ & $0.024$ & $0.543$ & $0.504$ \\
\midrule
RFM           & $0.292$ & $0.431$ & $0.835$ & $0.291$ & $0.278$ & $0.845$ & $0.307$ & $0.345$ & $0.790$ & $0.357$ & $0.176$ & $0.913$ \\
\midrule
GFM-L         & $\underline{0.099}$ & $0.516$ & $0.553$ & $0.040$ & $\mathbf{0.570}$ & $0.562$ & $0.062$ & $\underline{0.569}$ & $0.547$ & $0.029$ & $0.520$ & $0.643$ \\
GFM-E         & $0.199$ & $0.431$ & $0.823$ & $0.174$ & $0.339$ & $0.806$ & $0.191$ & $0.350$ & $0.795$ & $0.190$ & $0.201$ & $0.900$ \\
GFM-S         & $\mathbf{0.094}$ & $0.512$ & $0.677$ & $0.042$ & $0.478$ & $0.693$ & $0.060$ & $0.498$ & $0.651$ & $0.031$ & $0.392$ & $0.768$ \\
GFM-MF        & $0.119$ & $0.512$ & $0.711$ & $0.043$ & $0.510$ & $0.658$ & $0.059$ & $0.515$ & $0.642$ & $0.033$ & $0.408$ & $0.751$ \\
\midrule
RMF-L$v$      & $0.119$ & $0.504$ & $0.663$ & $0.041$ & $0.535$ & $0.640$ & $0.047$ & $0.545$ & $0.613$ & $0.034$ & $0.453$ & $0.715$ \\
RMF-E$v$      & $0.142$ & $0.496$ & $0.689$ & $0.038$ & $0.502$ & $0.653$ & $0.045$ & $0.494$ & $0.661$ & $0.032$ & $0.428$ & $0.740$ \\
RMF-S$v$      & $0.122$ & $0.512$ & $0.687$ & $0.040$ & $0.464$ & $0.703$ & $0.045$ & $0.498$ & $0.661$ & $0.033$ & $0.380$ & $0.770$ \\
\midrule
KGD-GG        & $0.237$ & $0.447$ & $0.799$ & $0.353$ & $0.252$ & $0.866$ & $0.163$ & $0.453$ & $0.700$ & $0.398$ & $0.154$ & $0.926$ \\
KGD-GL        & $0.254$ & $0.435$ & $0.732$ & $0.252$ & $0.353$ & $0.772$ & $0.158$ & $0.460$ & $0.657$ & $0.272$ & $0.291$ & $0.816$ \\
KGD-M         & $0.131$ & $0.488$ & $0.699$ & $0.051$ & $0.485$ & $0.689$ & $0.059$ & $0.521$ & $0.631$ & $0.138$ & $0.332$ & $0.801$ \\
\midrule
RCM-CT           & $0.147$ & $0.496$ & $0.691$ & $0.055$ & $0.520$ & $0.655$ & $0.052$ & $0.511$ & $0.659$ & $0.045$ & $0.405$ & $0.744$ \\
RCM-scT          & $0.142$ & $0.496$ & $0.811$ & $0.092$ & $0.376$ & $0.796$ & $0.115$ & $0.417$ & $0.751$ & $0.109$ & $0.236$ & $0.871$ \\
\midrule
\ourmethod-SG & $0.110$ & $\underline{0.533}$ & $0.563$ & $\mathbf{0.031}$ & $0.552$ & $0.565$ & $0.061$ & $0.558$ & $\underline{0.533}$ & $\mathbf{0.024}$ & $0.530$ & $0.643$ \\
\ourmethod-C  & $0.112$ & $\mathbf{0.545}$ & $0.612$ & $0.036$ & $0.538$ & $0.594$ & $0.053$ & $0.553$ & $0.580$ & $0.031$ & $0.475$ & $0.688$ \\
\ourmethod-G  & $0.103$ & $0.484$ & $\underline{0.510}$ & $0.047$ & $0.556$ & $\mathbf{0.520}$ & $\underline{0.044}$ & $0.551$ & $\mathbf{0.530}$ & $0.027$ & $\mathbf{0.554}$ & $\mathbf{0.583}$ \\
\ourmethod-H  & $0.114$ & $0.516$ & $0.587$ & $0.040$ & $0.527$ & $0.618$ & $0.052$ & $\mathbf{0.587}$ & $0.591$ & $\underline{0.025}$ & $0.426$ & $0.730$ \\
\ourmethod-M  & $0.122$ & $0.524$ & $0.638$ & $\underline{0.033}$ & $0.529$ & $0.624$ & $0.054$ & $0.522$ & $0.639$ & $\mathbf{0.024}$ & $0.439$ & $0.738$ \\
\ourmethod-S  & $\mathbf{0.094}$ & $0.512$ & $\mathbf{0.496}$ & $0.040$ & $\underline{0.568}$ & $\underline{0.537}$ & $\mathbf{0.043}$ & $0.563$ & $0.535$ & $0.028$ & $\underline{0.553}$ & $\underline{0.592}$ \\
\bottomrule
\end{tabularx}
\label{tab:sphere-results}
\vspace{-1.5\baselineskip}
\end{table*}

\subsection{Flat Torus -- $\mathbb{T}^d$}
\vspace{-0.5\baselineskip}
We also consider distributions on flat tori, the natural state spaces of torsion angles in protein and RNA that encode their structures. Specifically, we use protein side-chain torsion angles represented as points on $\mathbb{T}^2$~\citep{lovell2003structure}, and RNA backbone torsion angles on $\mathbb{T}^7$~\citep{murray2003rna}. The quantitative results are summarized in~\tabref{tab:torus-results}. Across nearly all settings, variants of~\ourmethod{} outperform both the Lagrangian and other self-distillation variants of GFM.
Interestingly, \ourmethod{}-SG performs comparably to \ourmethod{}-H, even though our framework
does not guarantee identifiability of the squared geodesic cost for every $\varepsilon > 0$.
This is consistent with \propref{prop:gaussian_kernels_almost_nondegenerate}: on $\mathbb{T}^d$ the
squared geodesic cost is identifiable for all but countably many $\varepsilon$, so a generic
choice of $\varepsilon$ behaves like an identifiable cost in practice.

\providecommand{\hdrtorus}{\rothdr{\kmmd$\downarrow$} & \rothdr{\cov$\uparrow$} & \rothdr{\nna}}

\begin{table*}[h]
\scriptsize
\setlength{\tabcolsep}{2pt}
\caption[Results on protein and RNA torsion-angle datasets]{\textbf{Results on protein and RNA torsion-angle datasets.} \kmmd{}, \cov{}, and \nna{} (closest to 0.5 is best) on four protein and RNA torsion-angle datasets (NFE\,$=$\,1; \textbf{bold} = best, \underline{underline} = second best; the held-out row is the val-vs-test reference and is excluded from ranking.) Variant suffixes as in~\tabref{tab:sphere-results}. 
The metrics are averaged over 3 runs. Full results with standard deviations are in~\tabref{tab:torus-results-full}.} 
\begin{tabularx}{\linewidth}{l*{15}{Y}}
\toprule
 & \multicolumn{3}{c}{General} & \multicolumn{3}{c}{Glycine} & \multicolumn{3}{c}{Proline} & \multicolumn{3}{c}{Prepro} & \multicolumn{3}{c}{RNA} \\
\cmidrule(lr){2-4} \cmidrule(lr){5-7} \cmidrule(lr){8-10} \cmidrule(lr){11-13} \cmidrule(lr){14-16}
Method
 & \hdrpaper
 & \hdrpaper
 & \hdrpaper
 & \hdrpaper
 & \hdrpaper \\
\midrule
\textit{Held-out} & $0.007$ & $0.585$ & $0.499$ & $0.023$ & $0.590$ & $0.491$ & $0.020$ & $0.602$ & $0.489$ & $0.027$ & $0.572$ & $0.514$ & $0.036$ & $0.588$ & $0.504$ \\
\midrule
RFM           & $0.469$ & $0.191$ & $0.831$ & $0.281$ & $0.314$ & $0.741$ & $0.478$ & $0.160$ & $0.880$ & $0.486$ & $0.225$ & $0.820$ & $0.559$ & $0.205$ & $0.942$ \\
\midrule
GFM-L         & $\mathbf{0.018}$ & $0.515$ & $0.561$ & $\underline{0.030}$ & $0.562$ & $\underline{0.523}$ & $0.026$ & $0.547$ & $0.545$ & $0.043$ & $0.551$ & $0.531$ & $0.102$ & $\underline{0.571}$ & $0.572$ \\
GFM-E         & $0.327$ & $0.257$ & $0.777$ & $0.134$ & $0.389$ & $0.681$ & $0.388$ & $0.162$ & $0.878$ & $0.324$ & $0.288$ & $0.794$ & $0.540$ & $0.313$ & $0.888$ \\
GFM-S         & $0.067$ & $0.362$ & $0.690$ & $0.056$ & $0.450$ & $0.633$ & $0.059$ & $0.378$ & $0.699$ & $0.080$ & $0.423$ & $0.672$ & $0.183$ & $0.473$ & $0.756$ \\
GFM-MF        & $0.094$ & $0.405$ & $0.654$ & $0.044$ & $0.500$ & $0.581$ & $0.094$ & $0.421$ & $0.659$ & $0.122$ & $0.469$ & $0.624$ & $0.558$ & $0.248$ & $0.926$ \\
\midrule
RMF-L$v$      & $0.059$ & $0.453$ & $0.610$ & $0.049$ & $0.494$ & $0.592$ & $0.063$ & $0.498$ & $0.586$ & $0.080$ & $0.507$ & $0.594$ & $0.195$ & $0.545$ & $0.669$ \\
RMF-E$v$      & $0.105$ & $0.386$ & $0.672$ & $0.047$ & $0.489$ & $0.604$ & $0.091$ & $0.401$ & $0.676$ & $0.106$ & $0.453$ & $0.643$ & $0.336$ & $0.449$ & $0.790$ \\
RMF-S$v$      & $0.037$ & $0.415$ & $0.643$ & $0.047$ & $0.474$ & $0.611$ & $0.031$ & $0.369$ & $0.701$ & $0.049$ & $0.511$ & $0.584$ & $0.094$ & $0.498$ & $0.644$ \\
\midrule
KGD-GG        & $0.181$ & $0.398$ & $0.656$ & $0.119$ & $0.459$ & $0.619$ & $0.029$ & $0.540$ & $0.554$ & $0.071$ & $0.552$ & $0.540$ & $0.560$ & $0.204$ & $0.948$ \\
KGD-GL        & $0.337$ & $0.283$ & $0.754$ & $0.213$ & $0.371$ & $0.693$ & $0.026$ & $\mathbf{0.576}$ & $0.526$ & $0.380$ & $0.294$ & $0.768$ & $0.559$ & $0.213$ & $0.940$ \\
\midrule
RCM-CT           & $0.155$ & $0.342$ & $0.707$ & $0.267$ & $0.284$ & $0.768$ & $0.138$ & $0.377$ & $0.715$ & $0.332$ & $0.287$ & $0.791$ & $0.416$ & $0.407$ & $0.809$ \\
RCM-sCT          & $0.402$ & $0.210$ & $0.819$ & $0.331$ & $0.243$ & $0.807$ & $0.464$ & $0.139$ & $0.904$ & $0.351$ & $0.283$ & $0.776$ & $0.544$ & $0.311$ & $0.889$ \\
\midrule
\ourmethod-SG & $0.021$ & $\underline{0.547}$ & $\underline{0.532}$ & $0.032$ & $\underline{0.570}$ & $\underline{0.523}$ & $\underline{0.023}$ & $0.571$ & $\mathbf{0.508}$ & $\mathbf{0.027}$ & $\mathbf{0.580}$ & $\mathbf{0.510}$ & $0.105$ & $0.551$ & $\underline{0.519}$ \\
\ourmethod-C  & $0.086$ & $0.483$ & $0.582$ & $0.038$ & $0.567$ & $0.535$ & $\mathbf{0.018}$ & $0.563$ & $\underline{0.524}$ & $\underline{0.042}$ & $\underline{0.568}$ & $\underline{0.516}$ & $\mathbf{0.046}$ & $0.565$ & $0.532$ \\
\ourmethod-G  & $0.022$ & $\mathbf{0.562}$ & $\mathbf{0.519}$ & $0.117$ & $0.472$ & $0.598$ & $0.029$ & $\underline{0.574}$ & $0.536$ & $0.110$ & $0.526$ & $0.562$ & $\underline{0.062}$ & $\mathbf{0.580}$ & $\mathbf{0.508}$ \\
\ourmethod-H  & $\underline{0.020}$ & $0.542$ & $0.533$ & $\mathbf{0.028}$ & $\mathbf{0.574}$ & $\mathbf{0.517}$ & $\underline{0.023}$ & $0.567$ & $0.532$ & $0.062$ & $0.562$ & $0.518$ & $0.104$ & $0.552$ & $0.520$ \\
\bottomrule
\end{tabularx}
\label{tab:torus-results}
\end{table*}
\vspace{-0.5\baselineskip}

\subsection{General Manifolds}
\vspace{-0.5\baselineskip}
The spectral costs discussed in~\secref{sec:spectral-cost} offer a key advantage over alternative choices: they are applicable to arbitrary compact manifolds. This allows~\ourmethod{} to extend naturally beyond specific manifold geometries. To demonstrate this capability, we consider the task of learning distributions on triangular meshes, following the setup of~\citet{chen2024rfm}. As shown in~\tabref{tab:mesh-results},~\ourmethod{} matches the target distributions comparably to RFM~\citep{chen2024rfm} using only a single step. In contrast, RFM requires 1,000 Euler integration steps (RFM$_{1000}$), resulting in a sampling time that is two orders of magnitude longer. Qualitative results from~\ourmethod{} are provided in~\figref{fig:mesh-teaser}.

\newcommand{\hdrmesh}{%
  \rothdr{\kmmd$\downarrow$} &
  \rothdr{\cov$\uparrow$} &
  \rothdr{\nna}}

\providecommand{\hdrkmmdcovnnatime}{\rothdr{\kmmd$\downarrow$} & \rothdr{\cov$\uparrow$} & \rothdr{\nna} & \rothdr{Time\,(s)$\downarrow$}}

\begin{table*}[h]
\scriptsize
\setlength{\tabcolsep}{2pt}
\caption[Results on mesh datasets]{\textbf{Results on mesh datasets.}~\kmmd{},~\cov{},~\nna{} (closest to 0.5 is best), and sampling time on the Bunny and Spot datasets (\textbf{bold} = best, \underline{underline} = second best; the held-out row is the val-vs-test reference and is excluded from ranking).
RFM is sampled with 1{,}000 steps of its projected Euler integrator (RFM$_{1000}$);
\ourmethod-H uses the heat cost and is sampled with NFE\,$=$\,1.}
\begin{tabularx}{\linewidth}{l*{16}{Y}}
\toprule
 & \multicolumn{4}{c}{Bunny ($k=10$)} & \multicolumn{4}{c}{Bunny ($k=50$)} & \multicolumn{4}{c}{Spot ($k=10$)} & \multicolumn{4}{c}{Spot ($k=50$)} \\
\cmidrule(lr){2-5} \cmidrule(lr){6-9} \cmidrule(lr){10-13} \cmidrule(lr){14-17}
Method
 & \hdrkmmdcovnnatime
 & \hdrkmmdcovnnatime
 & \hdrkmmdcovnnatime
 & \hdrkmmdcovnnatime \\
\midrule
\textit{Held-out} & $0.020$ & $0.583$ & $0.493$ & --- & $0.025$ & $0.583$ & $0.498$ & --- & $0.012$ & $0.587$ & $0.500$ & --- & $0.014$ & $0.584$ & $0.515$ & --- \\
\midrule
RFM$_{1000}$ & $\underline{0.048}$ & $\mathbf{0.569}$ & $\mathbf{0.514}$ & $\underline{49.5}$ & $\mathbf{0.028}$ & $\mathbf{0.565}$ & $\mathbf{0.516}$ & $\underline{48.7}$ & $\mathbf{0.016}$ & $\mathbf{0.561}$ & $\mathbf{0.522}$ & $\underline{54.8}$ & $\underline{0.023}$ & $\underline{0.540}$ & $\mathbf{0.549}$ & $\underline{54.7}$ \\
\ourmethod-H & $\mathbf{0.030}$ & $\underline{0.558}$ & $\underline{0.537}$ & $\mathbf{0.16}$ & $\underline{0.034}$ & $\underline{0.544}$ & $\underline{0.550}$ & $\mathbf{0.13}$ & $\underline{0.018}$ & $\underline{0.493}$ & $\underline{0.601}$ & $\mathbf{0.24}$ & $\mathbf{0.015}$ & $\mathbf{0.549}$ & $\underline{0.559}$ & $\mathbf{0.24}$ \\
\bottomrule
\end{tabularx}
\label{tab:mesh-results}
\end{table*}
\vspace{-0.5\baselineskip}

\vspace{-0.5\baselineskip}
\section{Conclusion}
\vspace{-0.5\baselineskip}

We propose~\ourmethod{}, a one-step generative model guided by the Wasserstein gradient flow of Sinkhorn divergence for learning distributions on compact Riemannian manifolds.~\ourmethod{} builds on a rigorous theoretical foundation for characterizing cost functions that guarantee identifiability of the induced velocity fields. Our main theoretical result establishes a broad class of such cost functions, providing principled choices for the training objective. Experiments further show that~\ourmethod{} surpasses recent few-step baselines under matched model capacity and training cost, demonstrating strong empirical performance alongside its theoretical guarantees.

\clearpage

\subsection*{AI use statement}

In this work, AI tools were used to help the authors in
proving mathematical claims, implementing methods, supporting
qualitative and thematic analysis, and providing feedback on
the
research methodology or experiments.
We have not used generative AI tools to assist in translation
and interpretation
of the results. The generation of synthetic datasets and the cleaning or reformatting of datasets were not applicable to this work.
We have reviewed all AI-assisted work. The authors have manually verified
and refined the code, mathematical claims, as well as parts of the manuscript that are written or edited using AI tools. We take responsibility for the final content of this work, including text, claims or artifacts produced with the aid of generative AI.

\subsection*{Ethics Statement}
All authors have read and adhere to the ICLR Code of Ethics. This work develops generative modeling methodology and its theory. It involves no human subjects, crowdsourcing, or personally identifiable information. All experiments use publicly available datasets. 

\subsection*{Reproducibility Statement}
To support reproducibility, we will release our code publicly upon acceptance, together with instructions for running it locally and the exact configuration files used for every experiment.

\bibliography{iclr2027_conference}
\bibliographystyle{iclr2027_conference}

\newpage

\appendix
\section*{Appendix}
This appendix is organized as follows. \appref{app:notation} summarizes the notation used throughout the paper, and \appref{app:proofs} contains the proofs deferred from the main text. \appref{app:other-energy-functionals} discusses alternative energy functionals, such as the squared Maximum Mean Discrepancy (MMD) and the KL divergence over smoothed densities as in \citet{esteban2026kernel}. \appref{app:spectral-kernels} examines the spectral kernels used in this work (Heat, Mat\'ern, and Subordinated Heat) in greater depth. Finally, \appref{app:experiment-details} details our experiments, including the experimental setup, training and inference algorithms, 
and full results with standard deviations.

\section{Notation}
\label{app:notation}

\tabref{tab:notation} collects the notation used throughout the paper.
For an in-depth introduction to Riemannian Geometry, refer to \citet{lee2018introduction} or
the appendix of \citet{woo2026riemannian}.
For optimal transport and Wasserstein gradient flows on Riemannian manifolds we refer to \cite{villani2008optimal}, and for the entropy-regularized transport and the Sinkhorn divergence to \citet{feydy2019interpolating}.

\footnotesize
\renewcommand{\arraystretch}{1.05}
\begin{longtable}{@{}l p{0.82\linewidth}@{}}
\caption{\textbf{Notations used in the paper}.}
\label{tab:notation} \\
\toprule
Symbol & Meaning \\
\midrule
\endfirsthead
\toprule
Symbol & Meaning \\
\midrule
\endhead
\bottomrule
\endlastfoot
\multicolumn{2}{@{}l}{\textbf{Manifold and geometry}} \\
$(\Mfld, g)$ & compact, connected, smooth Riemannian manifold with metric $g$; $\mathbb{S}^d$ and $\mathbb{T}^d$ denote the unit sphere and the flat torus of dimension $d$ \\
$\mathbb{R}^D$ & ambient Euclidean space in which $\Mfld$ is embedded for the network parametrization \\
$\Tang{x}$ & tangent space of $\Mfld$ at $x$; $\langle\cdot,\cdot\rangle_g$ and $\|\cdot\|_g$ are the metric and its norm on $\Tang{x}$ \\
$\Exp{x}$, $\Log{x}$ & exponential map $\Tang{x}\to\Mfld$ at $x$ and its inverse on the domain where it is defined \\
$\dg(x,y)$ & geodesic distance on $(\Mfld,g)$ \\
$\gradg$, $\divg$ & Riemannian gradient and divergence; $\gradg^{(1)}$ is the gradient with respect to the first argument of a function on $\Mfld\times\Mfld$ \\
$\volg$ & Riemannian volume measure \\
$P_x$ & orthogonal projection $\mathbb{R}^D\to\Tang{x}$ \\
$(\lambda_n,\phi_n)$ & eigenpairs of the Laplace--Beltrami operator, with $\{\phi_n\}$ orthonormal in $L^2(\Mfld)$ \\
$\hat k_j$ & spectral coefficient of a kernel $k$ diagonalized by $\{\phi_j\}$ (\lemref{lem:spectral}) \\
\midrule
\multicolumn{2}{@{}l}{\textbf{Measures and transport}} \\
$\Prob(\Mfld)$ & probability measures on $\Mfld$ \\
$\Meas_b(\Mfld)$ & finite signed measures on $\Mfld$; $|\sigma|$ is the total-variation measure of $\sigma$ \\
$p$, $q_0$, $q_\theta$ & target distribution, base distribution, and model distribution $q_\theta=(f_\theta)_{\#}q_0$, where $\#$ denotes the pushforward \\
$c(x,y)$ & cost function on $\Mfld\times\Mfld$, assumed symmetric and Lipschitz\\
$\varepsilon$ & entropic regularization strength\\
$k=e^{-c/\varepsilon}$ & Gibbs kernel of the cost $c$ \\
$k*\sigma$ & kernel convolution operator, $(k*\sigma)(x)=\int k(x,y)\,\dd\sigma(y)$ \\
$\Pi(q,p)$ & couplings with marginals $q$ and $p$ \\
$\OT_{\varepsilon,c}$, $\Sink_{\varepsilon,c}$ & entropy-regularized optimal transport cost \eqref{eq:eot} and Sinkhorn divergence \eqref{eq:sinkhorn} \\
$\pi^{qp}$ & optimal coupling of $\OT_{\varepsilon,c}(q,p)$; $\pi^{qp}(\dd y\mid x)$ is its disintegration with respect to $x$ \\
$u$, $v$, $w$ & dual potentials of $\OT_{\varepsilon,c}(q,p)$ in the first and second argument, and the symmetric dual potential of $\OT_{\varepsilon,c}(q,q)$ \\
$\KL$ & Kullback--Leibler divergence \\
$\gF$, $\frac{\delta\gF}{\delta q}$ & energy functional on $\Prob(\Mfld)$ and its first variation \\
$V_{q,p}$, $V^{\varepsilon,c}_{q,p}$ & velocity field of the Wasserstein gradient flow of $\gF$\\
$k_\rho$, $\rho$ & spectral kernel \eqref{eq:spectral_kernel} and its spectral density; $c_{k_\rho}=-\varepsilon\log k_\rho$ is the induced spectral cost \\
$\Gamma$ & set of $\varepsilon$ for which the geodesic Gaussian kernel is degenerate (\propref{prop:gaussian_kernels_almost_nondegenerate}) \\
\midrule
\multicolumn{2}{@{}l}{\textbf{Training}} \\
$f_\theta$, $v_\theta$ & one-step generator $\Mfld\to\Mfld$ and the velocity network $\Mfld\to\mathbb{R}^D$, with parameters $\theta$ \\
$z$ & latent sample, $z\sim q_0$ \\
$\eta$ & step size of the update $\Exp{x}(\eta V(x))$ \\
$\gL$ & training loss \eqref{eq:manifold-loss} \\
$\stopgrad$, $\mathrm{sg}$ & stop-gradient operator \\
$x_i$, $x'_l$, $y_j$ & generated samples, a second independent batch of generated samples, and data samples \\
$\hat q_\theta$, $\hat q'_\theta$, $\hat p$ & empirical measures of the three batches, with weights $a_i$, $a'_l$, $b_j$ \\
\midrule
\multicolumn{2}{@{}l}{\textbf{\appref{app:other-energy-functionals}}} \\
$\gH_k$ & reproducing kernel Hilbert space of a positive definite kernel $k$ \\
$m_\mu$ & kernel mean embedding of a measure $\mu$ \\
$\tilde p_k$ & kernel-smoothed density of $p$ \\
\end{longtable}
\normalsize

\section{Proofs}
\label{app:proofs}

Unless otherwise stated, we assume $\Mfld$ to be a compact, connected Riemannian manifold, $c$ to be symmetric and Lipschitz continuous, and set $k:=e^{-c/\varepsilon}$.

\subsection{Proof of Theorem~\ref{thm:main}}
\label{app:proof-main}
Here we derive the main formula following well-established results. The general cost function and the manifold setting barely change the computation from the Euclidean case.

\thmmain*

\begin{proof}
    Start by considering the dual problem of the entropy-regularized OT \citep{10.1561/2200000073},
    \begin{equation}
        \OT_{\varepsilon, c}(q,p)\;=\;\max_{u,v\in C(\Mfld)}\;
        \underbrace{\int u\dd q+\int v\dd p
        -\varepsilon\iint\Big(e^{(u\oplus v-c)/\varepsilon}-1\Big)\dd q\dd p}_{=:D(u,v;q,p)} .
        \label{eq:dual}
    \end{equation}
    where $(u \oplus v)(x, y) = u(x) + v(y)$, and the dual potentials $u$ and $v$ satisfy
    \begin{equation}
        u(x)=-\varepsilon\log\int_\Mfld e^{(v(y)-c(x,y))/\varepsilon}\dd p(y),
        \qquad
        v(y)=-\varepsilon\log\int_\Mfld e^{(u(x)-c(x,y))/\varepsilon}\dd q(x).
        \label{eq:softmin}
    \end{equation}
    Furthermore, the optimal transport plan solving \eqref{eq:eot} is given by
    $\pi = \exp \big(
    \frac1\varepsilon (u\oplus v-c)\big)\cdot (q\otimes p)
    $
    and that $\int_\Mfld \exp \big(
    \frac1\varepsilon (u\oplus v-c)\big)\dd p(y)=1$
    for $p$-a.e. $x.$
    Then the envelope theorem lets us differentiate at the optimum holding
    $(u, v)$ fixed:
    \begin{align}
        \int\frac{\delta\OT_{\varepsilon, c}(\cdot, p)}{\delta q}\dd\chi
        &= \frac{\dd}{\dd s}\Big|_{0}D(u,v;q+s\chi,p)\\
        &=\int u\dd\chi-\varepsilon\int\Big[\int e^{(u\oplus v-c)/\varepsilon}\dd p(y)-1\Big]\dd\chi(x)\\
        &=\int u\dd\chi
    \end{align}
    for a signed measure $\chi$ with $\chi(\Mfld) = 0.$
    Thus, the first variation of the Sinkhorn divergence is
    \begin{equation}
        \frac{\delta \Sink_{\varepsilon, c}(\cdot, p)}{\delta q}
        = u - w,
    \end{equation}
    where $w$ is the symmetric dual potential of $\OT_{\varepsilon, c}.$
    Now taking the Riemannian gradient, we have
    \begin{align}
        \gradg u(x)
        &=-\varepsilon\cdot
        \frac{\displaystyle\int e^{(v(y)-c(x,y))/\varepsilon}\Big(-\tfrac1\varepsilon\gradg^{(1)}c(x,y)\Big)\dd p(y)}
             {\displaystyle\int e^{(v(y)-c(x,y))/\varepsilon}\dd p(y)}\\
        &=\frac{\displaystyle\int e^{(v(y)-c(x,y))/\varepsilon}\,\gradg^{(1)}c(x,y)\dd p(y)}
             {\displaystyle\int e^{(v(y)-c(x,y))/\varepsilon}\dd p(y)}
        = \int \gradg^{(1)} c(x, y) \pi (\dd y \mid x)
    \end{align}
    where the third equality comes from multiplying the numerator and denominator by \(e^{u(x)/\varepsilon}\)
    as
    the denominator becomes 1, whereas the numerator
    becomes an integral against \(e^{(u(x)+v(y)-c(x, y))/\varepsilon}\dd p(y)
    =\pi(\dd y\,|\, \dd x)\). Hence the final result,
    \begin{equation}
        \gradg \frac{\delta \Sink_{\varepsilon, c}(\cdot, p)}{\delta q}
        = \gradg u - \gradg w
        =\int \gradg^{(1)}c(x,y)\;\pi^{qp}(\dd y\mid x)
        -
        \int \gradg^{(1)}c(x,z)\;\pi^{q q}(\dd z\mid x).
    \end{equation}
\end{proof}

\subsection{Proof of Theorem \ref{thm:identifiable}}

We prove the characterization of identifiable
velocity fields under mild conditions.

\thmidentifiable*

We first prove the following useful lemma.

\begin{lemma}[Zero velocity is a kernel equation]
\label{lem:kernel_eq}
For $p,q\in\mathcal{P}(\Mfld)$,
\[
V^{\varepsilon,c}_{q,p}=0 \ \text{ on } \Mfld
\iff
k*\big(e^{w/\varepsilon}(p-q)\big)=0 \ \text{on } \Mfld .
\]
\end{lemma}
\begin{proof}
    By the proof of \thmref{thm:main}, $\frac{\delta \Sink (\cdot, p)}{\delta q} = u-w$, where $u$ and $w$ are the first-argument dual potentials of $\OT_{\varepsilon,c}(q,p)$ and $\OT_{\varepsilon,c}(q,q)$, respectively, and
    \begin{equation}
        V^{\varepsilon, c}_{q, p} = -\gradg \frac{\delta \Sink (\cdot, p)}{\delta q} = \gradg (w - u)=0
    \end{equation}
    if and only if $w - u = K$ for some constant $K\in\mathbb{R}$ since $\Mfld$ is connected. 
    
    \medskip
    ($\Rightarrow$) As the dual potentials are unique up to constants \citep{feydy2019interpolating} we can consider the shifted potentials
    $(u+K, v-K)$ and assume $w-u=0.$ The dual potentials satisfy (\eqref{eq:softmin})
    \begin{align}
        e^{-u(x)/\varepsilon} &=  \int e^{(v(y)-c(x, y))/\varepsilon}\dd p(y)\\
        &=\int e^{-c(x, y)/\varepsilon}e^{v(y)/\varepsilon}\dd p(y) = (k * (e^{v/\varepsilon} p))(x).
    \end{align}
    Analogously, since $c(x, y)$,
    and hence $k(x, y)$, is symmetric,
    \begin{align}
        e^{-v(y)/\varepsilon}
        &=\int e^{(u(x) - c(x, y))/\varepsilon}\dd q(x)
        =\int k(x, y) e^{u(x)/\varepsilon}\dd q(x)\\
        &=\int k(y, x) e^{u(x)/\varepsilon}\dd q(x)
        = (k * (e^{u/\varepsilon}q))(y)\\
        e^{-w(x)/\varepsilon} &= (k * (e^{w/\varepsilon} q))(x)
    \end{align}
    Furthermore, $u=w$ implies
    \begin{equation}
        \label{eq:kernel_null}
        e^{-u/\varepsilon} - e^{-w/\varepsilon} = k * (e^{v/\varepsilon} p) - k * (e^{w/\varepsilon} q) 
        = k * (e^{v/\varepsilon} p - e^{w/\varepsilon} q) = 0,
    \end{equation}
    as $*$ is a linear operation. Moreover, we have 
    \begin{align}
        e^{-v(y)/\varepsilon} &= (k * (e^{u/\varepsilon}q))(y)
        = (k * (e^{w/\varepsilon}q))(y)\\
        &= \int k(y, x) e^{w(x)/\varepsilon}\dd q(x)
        = e^{-w(y)/\varepsilon}\quad
    \end{align}
    Hence, we have $v = w$ as well. This makes \eqref{eq:kernel_null}
    \begin{equation}
        k * (e^{w/\varepsilon}(p-q)) = 0.
    \end{equation}

    \medskip
    ($\Leftarrow$) Suppose $k*\big(e^{w/\varepsilon}(p-q)\big)=0.$ The dual $w$ satisfies
    \begin{equation}
        e^{-w/\varepsilon}=k*(e^{w/\varepsilon}q)=k*(e^{w/\varepsilon}p).
    \end{equation}
    By Prop. 11 of \citet{feydy2019interpolating}, $(w, w)$ solves the $\OT_{\varepsilon, c}(q, p)$
    and $w = u +K$ for some constant $K\in\mathbb{R}.$
\end{proof}

\begin{proof}[Proof of Theorem \ref{thm:identifiable}]
    ($\Leftarrow$) Suppose $k$ is nondegenerate. By \lemref{lem:kernel_eq} and the fact that $k$ is nondegenerate, 
    \begin{equation}
        k*\big(e^{w/\varepsilon}(p-q)\big)=0 \Rightarrow e^{w/\varepsilon}(p-q)=0\Rightarrow p = q
    \end{equation}
    as $e^{w/\varepsilon} > 0.$
    
    \medskip
    ($\Rightarrow$) To prove the contrapositive, suppose $k$ is degenerate. There exists
    a finite signed measure $\sigma\neq 0$ such that $k*\sigma=0$. Let $q = (|\sigma| +\mathrm{vol}) / Z\in\Prob(\Mfld)$ where
    $Z = \mathrm{vol}(\Mfld) + |\sigma| (\Mfld).$ Then define $p:= q + te^{-w/\varepsilon}\sigma$ and notice
    \begin{align}
        \int e^{-w/\varepsilon}d\sigma &= \int k * (e ^{w/\varepsilon}q)\dd\sigma
        =\iint k(x,y)e^{w(y)/\varepsilon}\dd q(y) \dd\sigma (x)\\
        &
        = \iint k(x,y)\dd \sigma(x) e^{w(y)/\varepsilon}\dd q(y)
        = \int (k * \sigma) e^{w/\varepsilon}\dd q(y) = 0.
    \end{align}
    Since $e^{-w/\varepsilon}$ is continuous on a compact $\Mfld,$ it is bounded, so
    we can choose small $t > 0$ such that $t\sup e^{-w/\varepsilon} < \frac1Z$ and $p = q + t e^{-w/\varepsilon}\sigma
    =\frac1Z|\sigma| + \frac1Z\mathrm{vol} + te^{-w/\varepsilon}\sigma>\frac{1}{Z}\mathrm{vol}.$ %
    Hence, we have constructed $p, q\in \Prob(\Mfld)$ such that $k * \big(e^{w/\varepsilon}(p - q)\big)=0.$
    By \lemref{lem:kernel_eq}, $V_{q, p}^{\varepsilon, c}=0$ but $q\neq p$, so $V_{q, p}^{\varepsilon, c}$ is not identifiable.
\end{proof}

\subsection{Proof of Lemma~\ref{lem:spectral}}

\lemspectral*

\begin{proof}
    ($\Rightarrow$) Suppose $\hat k_j = 0$ for some $j.$ Then, 
    \begin{equation}
        k * (\phi_j\volg )=\int k(x, y) \phi_j(y)\dd \volg(y)=\hat k_j \phi_j = 0.
    \end{equation}
    $k$ is degenerate.

    \medskip
    ($\Leftarrow$) Suppose $\hat k_j \neq 0$ for all $j.$ Let $\sigma\in \Meas_b(\Mfld)$ be a finite signed measure such that $k * \sigma = 0.$ By Fubini and the symmetry of $k,$
    \begin{align}
        0&=\int (k * \sigma) \phi_j \dd \volg= \iint k(x, y)\dd \sigma(x) \phi_j(y)\dd \volg(y)\\
        &= \iint k(y, x) \phi_j(y)\dd\volg(y)\dd\sigma(x)=\int \hat k_j \phi_j\dd\sigma = \hat k_j \int \phi_j\dd\sigma
    \end{align}
    Thus $\int \phi_j\dd\sigma=0$ for every eigenfunction $\phi_j.$ The
    span of eigenfunctions is dense in $C^\infty(\Mfld)$ \citep{berger1971spectre}, which is
    dense in $C(\Mfld)$. Now Hahn-Banach density theorem and Riesz Representation theorem implies $\sigma =0,$ so $k$ must be nondegenerate.
\end{proof}

\subsection{Proof of Proposition~\ref{prop:gaussian_kernels_almost_nondegenerate}}
\label{app:proof-sd-td}

\propgaussnondeg*

We prove the above for $\mathbb{S}^d$
($d\ge 2$) and $\mathbb{T}^d$ ($d\ge 1$); the case
$\mathbb{S}^1$ will be covered by the second proposition.

\begin{proposition}[Sphere]
\label{prop:sphere}
Let $d\ge2$ and $k_\varepsilon = e^{-d_g^2/2\varepsilon}$ on $\mathbb{S}^d$. 
The set $\Gamma$ of $\varepsilon>0$ for which $k_\varepsilon$ is degenerate is at most countable.
\end{proposition}

\begin{proof}
    Set $\beta := 1/(2\varepsilon)$ and $k_\beta := e^{-\beta d_g^2}$. 

    Since $d_g(x,y)=\arccos\langle x,y\rangle$, the Funk-Hecke formula \citep{andrews1999special} gives, for every 
    spherical harmonic $Y_\ell$ of degree $\ell$ and with $t=\cos\theta$,
    \begin{align}
        \label{eq:funk-hecke}
        k_\beta * (Y_\ell \volg) = \gamma_\ell(\beta)\, &Y_\ell, \qquad
        \gamma_\ell(\beta) = |\mathbb{S}^{d-1}| \int_0^{\pi} e^{-\beta\theta^2}\, g_\ell(\theta)\,\dd\theta,\\
        g_\ell(\theta) :&= P_{\ell,d+1}(\cos\theta)\,\sin^{d-1}\theta,
    \end{align}
    where $P_{\ell, d+1}$ is the Legendre polynomial of degree $\ell$ in dimension $d+1$.
    Spherical harmonics are continuous, 
    so by~\lemref{lem:spectral}, $k_\beta$ is nondegenerate if and only if $\gamma_\ell(\beta)\neq0$ for all $\ell\ge0$.

    Since the integrand is entire in $\beta$
    and jointly continuous on $\mathbb{C}\times [0, \pi]$,
    $\gamma_\ell$ is entire.

    Notice that $\gamma_\ell$ is the Laplace-transform of the function
    $G_\ell(s):=\frac{g_\ell(\sqrt s)}{2\sqrt s}1_{[0, \pi^2]}(s)$ with $s=\theta^2.$
    By Lerch's theorem, $\gamma_\ell\not\equiv 0$ since $G_\ell(s)$ is not an identically
    zero function.

    Since $\gamma_\ell\not\equiv 0$, by the identity theorem, $Z_\ell:=\{\beta>0:\gamma_\ell(\beta)=0\}$ is 
    at most countable. By \eqref{eq:funk-hecke} and \lemref{lem:spectral}, $k_\beta$ is degenerate exactly for 
    $\beta\in\bigcup_{\ell\ge0}Z_\ell$, which is countable.
\end{proof}

\begin{proposition}[Torus]
\label{prop:torus}
Let $d\ge1$ and $k_\varepsilon = e^{-d_g^2/2\varepsilon}$ on $\mathbb{T}^d=(\mathbb{R}/2\pi\mathbb{Z})^d$. 
The set $\Gamma$ of $\varepsilon>0$ for which $k_\varepsilon$ is degenerate is countably infinite.
\end{proposition}

\begin{proof}
    Set $\beta:=1/(2\varepsilon)$ and $k_\beta:=e^{-\beta d_g^2}$ as before.

    Here $d_g^2(x,y)=\sum_{i=1}^d|x_i-y_i|_{\mathrm{w}}^2$, with $|\cdot|_{\mathrm{w}}$ the difference 
    wrapped to $[-\pi,\pi]$, so $k_\beta(x,y)=\prod_i e^{-\beta|x_i-y_i|_{\mathrm{w}}^2}$. For $n\in\mathbb{Z}^d$, 
    Fubini's theorem and the substitution $y_i=x_i-\theta_i$ (after centering $x-\pi \le y\le x+\pi$) give
    \begin{align}
        \int_{\mathbb{T}^d}k_\beta(x,y)\,e^{in\cdot y}\,\dd y
        &= e^{in\cdot x}\prod_{i=1}^d\int_{-\pi}^{\pi}e^{-\beta\theta^2}e^{-in_i\theta}\,\dd\theta
        = e^{in\cdot x}\prod_{i=1}^d h_{n_i}(\beta),\\
        h_m(\beta):&=\int_{-\pi}^{\pi}e^{-\beta\theta^2}e^{im\theta}\,\dd\theta.
    \end{align}
    Taking real and imaginary parts, $\cos(n\cdot x)$ and $\sin(n\cdot x)$ are 
    eigenfunctions of $k_\beta*$ with eigenvalue $\prod_i h_{n_i}(\beta)$. These functions are continuous. Since $h_0>0$ and $h_{-m}=h_m$, 
    Lemma~\ref{lem:spectral} shows that $k_\beta$ is nondegenerate if and only if $h_m(\beta)\neq0$ for all $m\ge1$.

    Substitute $s= \theta\sqrt \beta + \frac{mi}{2\sqrt \beta}$, and complete the square of the integrand exponent,
    \begin{equation}
        h_m(\beta)=\frac{e^{-\frac{m^2}{4\beta}}}{\sqrt\beta}\int_{-\bar z}^z e^{-s^2} \dd s= \frac12\sqrt \frac{\pi}{\beta} e^{-\frac{m^2}{4\beta}} \big(\erf(z) - \erf(-\bar z)\big) = \sqrt\frac{\pi}{\beta} e^{-\frac{m^2}{4\beta}} \Real \erf (z),
    \end{equation}
    where $z = \pi\sqrt\beta + \frac{mi}{2\sqrt\beta}.$ Let $H_m(\beta):=\Real \erf (z(\beta))$, and since $z^2=\pi^2\beta-\frac{m^2}{4\beta}+i\pi m$,
    \begin{equation}
        H_m'(\beta)=\sqrt{\pi\over\beta}(-1)^m e^{\frac{m^2}{4\beta}-\pi^2\beta}.
    \end{equation}
    Therefore, $H_m(\beta)$ is a monotone function and $H_m(\beta)\to1$ as $\beta\to\infty$,
    \begin{equation}
        H_m(\beta)=1 - (-1)^m\int^{\infty}_{\beta} \sqrt{\pi\over s}e^{m^2/4s - \pi^2s}\dd s =: 1 - (-1)^m J_m(\beta)
    \end{equation}
    When $m$ is odd, $H_m(\beta) > 1$. When $m$ is even, $H_m$ increases strictly from 
    $-\infty$ to $1$ due to $J_m(\beta)\to\infty$
    as $\beta \to0^+.$ Hence, $H_m(\beta)$ has exactly one zero for even $m$, which in turn leaves
    $h_m(\beta)$ with exactly one zero.
    
    This implies $|Z_m|:=|\{\beta>0:h_m(\beta)=0\}|=0 \text{ or }1$ depending on
    the parity of $m.$ Moreover, for different $m > m'$, $J_m(\beta) > J_{m'}(\beta)$, so each $h_m$
    has unique $\beta_m$ as their root.

    $k_\beta$ is degenerate exactly for $\beta\in\bigcup_{m\ge1}Z_m$, which is countably infinite set.
\end{proof}

\section{Other Energy Functionals}
\label{app:other-energy-functionals}

Although our analysis has focused on the Sinkhorn divergence,
other energy functionals can also serve as the guiding discrepancy field. 
In this section, we describe how the framework changes under alternative 
choices of energy functional and compare their resulting forms and
identifiability conditions. 
\tabref{tab:energy-functionals} summarizes the energy functionals, their velocity fields, and their identifiability conditions.

We begin by fixing the relevant definitions.

\begin{definition}[Reproducing Kernel Hilbert Space (RKHS)]
    Let $\Mfld$ be a compact Riemannian manifold.
    Let $k:\Mfld\times\Mfld\to\mathbb R$ be a symmetric positive definite kernel.
    By the Moore-Aronszajn theorem \citep{aronszajn1950theory}, there exists a
    unique Hilbert space $\mathcal H_k$ of functions on $\Mfld$, namely the completion of
    $\operatorname{span}\{k(x,\cdot) : x\in\Mfld\}$, such that $k(x,\cdot)\in\mathcal H_k$
    for all $x\in\Mfld$ and the reproducing property holds:
    \begin{equation}
        \langle f, k(x,\cdot)\rangle_{\mathcal H_k} = f(x)
        \qquad \forall f\in\mathcal H_k,\; x\in\Mfld.
    \end{equation}
\end{definition}

\begin{definition}[Universal kernel]
    \label{def:universality}
    Let $\Mfld$ be a compact Riemannian manifold. A continuous positive definite kernel $k$
    is \emph{universal} if $\mathcal H_k$ is dense in $\gC(\Mfld)$ with respect to
    $\|\cdot\|_\infty$ \citep{steinwart2001influence}.
\end{definition}

\begin{definition}[Characteristic kernel]
    \label{def:charcteristic}
    A measurable positive definite kernel $k$ is \emph{characteristic} if the
    kernel mean embedding
    \begin{equation}
        \mu \mapsto m_\mu := \int_{\Mfld} k(\cdot, y)\,\dd\mu(y)
    \end{equation}
    is injective on $\mathcal P(\Mfld)$, where
    $\int \sqrt{k(y,y)}\,\dd\mu(y) < \infty$ is assumed so that
    $m_\mu \in \mathcal H_k$ \citep{smola2007hilbert}.
\end{definition}

\begin{remark}
    The definition of characteristic kernel is similar to the definition of nondegeneracy, but contrary to characteristic
    kernels,
    nondegenerate kernels need not be positive definite. Moreover, characteristic kernels operate on
    $\Prob(\Mfld)$, whereas nondegeneracy is defined on $\Meas_b(\Mfld).$
\end{remark}

\subsection{Squared Maximum Mean Discrepancy (MMD)}

\begin{definition}
    The squared Maximum Mean Discrepancy (MMD) between two distributions $p$ and $q$ is given by
    $\tfrac12\,\mathrm{MMD}^2(q, p) = \frac12\|m_p - m_q\|_{\gH}^2= \frac12\iint k(x, y)\dd r(x) \dd r(y)$ for $r = p - q$ and $p, q\in \Prob(\Mfld).$
\end{definition}

\begin{proposition}
    The velocity field $V_{q, p}$ induced by $\tfrac12\,\mathrm{MMD}^2$ is identifiable if
    and only if $k$ is characteristic.
\end{proposition}

\begin{proof}
    Let $\gF_{\mathrm{MMD}}(q) = \frac{1}{2}\mathrm{MMD}^2(q, p).$
    It can be shown \citep{han2026wflow} that
    \begin{equation}
        \frac{\delta \gF_{\mathrm{MMD}}}{\delta q}(q) = m_q - m_p
    \end{equation}
    Hence, $V_{q, p}$ has the form 
    \begin{equation}
        V_{q, p} = -\gradg (m_q - m_p) = \int\gradg k(x, y) \dd p(y) -\int \gradg k(x, y) \dd q(y). %
    \end{equation}

($\Leftarrow$) Suppose $k$ is characteristic. We have $V_{qp}\equiv 0 \Rightarrow m_q - m_p \equiv \text{const.}$
Letting $r = q - p,$
\begin{equation}
    \|m_q - m_p\|^2_\gH = \|m_r\|^2_\gH =  \iint k(x, y) \dd r(x) \dd r(y) = \int m_r \dd r(y) = m_r\int \dd r(y) = 0
\end{equation}
since $r$ has a zero mass on the space being the difference of the two probability measures. This gives
$m_p = m_q$ and as the mean embedding is injective, we must have $p=q.$

($\Rightarrow$) Suppose identifiability and $m_q=m_p$ for some $q,p\in\Prob(\Mfld).$ We further have
\begin{equation}
    V_{q,p}=-\gradg (m_q-m_p) \equiv 0 \;\Rightarrow\; q=p.
\end{equation}
This establishes $k$ is characteristic.

\end{proof}

\subsection{Smoothed KL Divergence}

This case is treated by \citet{esteban2026kernel}, but for the sake of completeness, we include
the details here.

\begin{definition}[Smoothing]
    Let $k$ be a positive normalized kernel such that $\int k(x, y)\dd\volg(x)=1.$ Denote the kernel-smoothed density by
    \begin{equation}
        \Tilde{p}(x) = \int k(x, y)\dd p(y).
    \end{equation}
\end{definition}

\begin{remark}
    As a function, 
    $\Tilde{p}$ is the same as the mean embedding of $p$ to RKHS of $k$. However, here
    we treat $\Tilde{p}(x)$ as a density of a probability measure with respect
    to the base measure $\dd\volg$, which is true since
    the kernel is normalized.
\end{remark}

In contrast with our framework, \citet{esteban2026kernel}
considers a Wasserstein Gradient Flow from $\Tilde q$. In particular, 
the Wasserstein Gradient Flow corresponding to their setting is
\begin{equation}
    \label{eq:kgd_wgf}
    \partial_t r_t = \divg \left(r_t \gradg\frac{\delta\gF^{KL}}{\delta r}(r_t)\right),\quad r_0:=\Tilde q
\end{equation}
where $\gF^{KL}(q):=\KL(q \parallel \Tilde p)$, i.e., their
guiding functional is the regular KL divergence between the smoothed densities $\Tilde q$
and $\Tilde p.$

\begin{proposition}
    Let $k$ be a normalized kernel such that $k > 0$.
    The velocity field $V_{\Tilde q,\Tilde p}$ induced by $\gF^{KL}$ is identifiable if and only if
    $k$ is characteristic.
\end{proposition}

\begin{proof}
    Define $\gF^{\mathrm{KL}}$ as in the above. It can be shown that \citep{han2026wflow}
    \begin{equation}
        \frac{\delta \gF^{\mathrm{KL}}}{\delta q} (\Tilde{q}) = -\log \Tilde{p} + \log \Tilde{q}.
    \end{equation}
    Thus, we get
    \begin{equation}
        V_{\Tilde q, \Tilde p}(x) = - \gradg \frac{\delta \gF^{\mathrm{KL}}}{\delta q} (\Tilde{q})(x) =\gradg \log \frac{\Tilde{p}(x)}{\Tilde{q}(x)}.
    \end{equation}
    ($\Leftarrow$) Suppose $k$ is characteristic. This means when $V_{q, p}\equiv0,$
    \begin{equation}
        \gradg \log\Tilde{p}(x)=\gradg \log\Tilde{q}(x).
    \end{equation}
    Since $k$ is a normalized kernel,
    \begin{equation}
        \Tilde{p}(x)=\Tilde{q}(x).
    \end{equation}
    As $k$ is characteristic, we have $q=p.$
    
    ($\Rightarrow$) Assume identifiability. Let $p$, $q$ be probability measures such that
    $\Tilde{p}=\Tilde{q}.$ Then, 
    \begin{equation}
        V_{\Tilde q,\Tilde p}=\gradg \log \frac{\Tilde{p}(x)}{\Tilde{q}(x)}=0.
    \end{equation}
    By identifiability, $q=p,$ establishing $k$ is characteristic.
    
\end{proof}

\begin{table}[h]
  \centering
  \small
  \renewcommand{\arraystretch}{1.8}
  \renewcommand{\tabularxcolumn}[1]{m{#1}} %
  \caption{\textbf{Summary of velocity terms and identifiability conditions for different energy functionals.}
  The velocity field of an energy functional is $V_{q,p} = T_{q,p} - T_{q,q}$.}
  \begin{tabularx}{\linewidth}{@{} l >{\centering\arraybackslash}X >{\centering\arraybackslash}m{0.3\linewidth} @{}}
    \toprule
    \textbf{Energy Functional} & $\boldsymbol{T_{q,p}(x)}$ & \textbf{Identifiability Condition} \\
    \midrule
    Sinkhorn
      & $\displaystyle -\int \gradg^{(1)} c(x,y)\, \pi^{qp}(\dd y \mid x)$
      & $k$ nondegenerate, $c$ symmetric Lipschitz \\
    MMD
      & $\displaystyle \int \gradg^{(1)} k(x,y)\, \dd p(y)$
      & $k$ characteristic \\
    Smoothed KL (KGD)
      & $\displaystyle \frac{\int \gradg^{(1)} k(x,y)\, \dd p(y)}{\int k(x,y)\, \dd p(y)}$
      & $k$ characteristic, normalized \\
    \bottomrule
  \end{tabularx}
      \label{tab:energy-functionals}
\end{table}

We summarize the relationships between proposed energy functionals in~\tabref{tab:energy-functionals}. Overall,
both $\frac12\mathrm{MMD}^2$ and smoothed KL divergence can achieve identifiability whenever $k$ is characteristic,
whereas the Sinkhorn divergence requires the Gibbs kernel to be nondegenerate (\thmref{thm:identifiable}), which for positive definite kernels is universality (\propref{prop:uni_nondeg}). \cite{smola2007hilbert} proves that if $k$
is a universal kernel, it must be characteristic.

\section{Discussion on Spectral Kernels}
\label{app:spectral-kernels}

Recall that given a compact manifold $\Mfld$ and its eigenpairs $\{\lambda_n, \phi_n\}_{n\ge 0}$
of the Laplace-Beltrami operator such that they form an orthogonal basis for $L^2(\Mfld)$, 
the spectral kernels are defined as
\begin{equation}
    k_\rho(x, y)=\sum_{n=0}^\infty \rho(\lambda_n)\phi_n(x)\phi_n(y),
\end{equation}
where $\rho(\lambda_n)$ is the corresponding spectral density.

\paragraph{Heat Kernel.} If we set $\rho_H(\lambda)=e^{-\lambda t}$,
the resulting expression is the heat kernel, the fundamental solution of the heat
equation, 
\begin{equation}
    \frac{\partial }{\partial t}K(x, y, t) = \Delta_g^{(1)} K(x, y, 0), \quad K(x, y, 0)=\delta(x - y)
\end{equation}
where $\Delta_g^{(1)}$ is the Laplace-Beltrami operator acting on the second argument,
and the initial condition $\delta(x-y)$ is the Dirac delta function.

Heat kernels are always positive due to the strong maximum/minimum principle. Moreover, the score
of heat kernel is bounded \citep{elton1999estimates}, 
\begin{equation}
    \|\gradg^{(1)} \log K(x, y, t)\|_g \le C_1\bigg(\frac{d_g(x, y)}{t} + \frac{1}{\sqrt t}\bigg) < \infty
\end{equation}
and hence $\log K(x, y, t)$ is a Lipschitz function with respect to $x$ and $y$. This is why heat kernel
is a valid kernel in the lens of \thmref{thm:identifiable}.

\paragraph{Mat\'ern Kernel.}
When $\rho_M(\lambda)=\sigma^2(\frac{2\nu}{\kappa^2}+\lambda)^{-\nu-d/2}$, we obtain
the Mat\'ern kernel \citep{borovitskiy2020matern} with the three parameters $\sigma^2$, $\nu$ and $\kappa.$ Here $d$
is the dimension of the underlying space. We may relate Mat\'ern kernels to heat kernels by
\begin{equation}
    \mathrm{Mat\Acute{e}rn(\sigma^2, \nu, \kappa^2)}\to \mathrm{Heat(\kappa^2/2)} \quad \text{ as }\quad \nu\to\infty,
\end{equation}
that is, the heat kernel can be viewed as the limit of the \textrm{Mat\'ern} kernel.

\textrm{The Mat\'ern} kernels are strictly positive because they have a strict positive
gamma mixture representation \citep{azangulov2024stationary}:
\begin{equation}
    Q(x, y, t):=k_{\rho_M}(x, y) = \sigma^2 \int^\infty_0 t^{\nu-1+d/2}e^{-\frac{2\nu}{\kappa^2}t}K(x, y, t)\dd t > 0.
\end{equation}
Recalling the following upper bound for heat kernel \citep{grigoryan2009heat},
\begin{equation}
    K(x, y, t) \le C_2t^{-d/2}e^{-d_g^2/5t},
\end{equation}
we may compute the following bound for the gradient of the heat kernel
\begin{equation}
    |\gradg^{(1)} K(x, y, t)|=K(x, y, t)|\gradg^{(1)}\log K(x, y, t)|\le
    C_1C_2 t^{-(d+1)/2}\left(r + 1\right)e^{-r^2/5} \le C_3 t^{-(d+1)/2},
    \label{eq:score-heat-bound}
\end{equation}
where $r=\frac{d_g}{\sqrt{t}}$. Furthermore, letting $Q_t:= Q(t, \cdot, \cdot),$
\begin{align}
    \gradg^{(1)} Q_t(x, y) &= \frac{\sigma^2}{C_{\nu, \kappa}} \int^\infty_0t^{\nu -1+d/2}e^{-\frac{2\nu}{\kappa^2}t}\gradg^{(1)} K(x, y, t)\dd t\\
    &\le C_4\int^\infty_0t^{\nu -3/2}e^{-\frac{2\nu}{\kappa^2}t} \dd t
\end{align}
The RHS converges exactly when $\nu -3/2 > -1\Leftrightarrow \nu > 1/2$ by the Gamma
function's integral representation. As $Q_t$ is a positive continuous function, it attains a
nonzero minimum value $m_1$ on the compact set $\Mfld\times \Mfld$. We finally have 
\begin{equation}
    |\gradg^{(1)} \log Q_t| \le \frac{\sup |\gradg^{(1)}Q_t|}{m_1} < \infty.
\end{equation}
Therefore, $Q_t$ is indeed Lipschitz continuous for $\nu > 1/2.$ This confirms Mat\'ern
kernel is a valid kernel for \thmref{thm:identifiable}.

\paragraph{Subordinated Heat Kernel.}
When $\rho_S(\lambda)=e^{-t\lambda^{\alpha}}$ with $t>0$ and $\alpha\in(0,1]$, the
resulting kernel is the \emph{subordinated heat kernel}
\citep{bochner2005harmonic}. Let $(S_t)_{t\ge0}$ be the increasing L\'evy
process with $\mathbb{E}[e^{-\lambda S_t}]=e^{-t\lambda^{\alpha}}$, and
subordination of the heat semigroup \citep{schilling2026bernstein} gives the
mixture representation
\begin{align}\label{eq:subord}
  P(x,y, t)&=\int_0^{\infty}K(x,y, s)\,\dd\zeta_t(s)
  =\sum_{n\ge 0} \left(\int^\infty_0e^{-s\lambda_n}\dd\zeta_t(s)\right)\phi_n(x)\phi_n(y)\\
  &=\sum_{n\ge0}\mathbb{E}[e^{-\lambda_nS_t}]\phi_n(x)\phi_n(y)
  =\sum_{n\ge0}e^{-t\lambda_n^{\alpha}}\phi_n(x)\phi_n(y),
\end{align}
where $\zeta_t$ is the law of $S_t.$
Thus
$P_t:=P(t,\cdot,\cdot)$ is a positive mixture of heat kernels. Strict positivity of 
heat kernel $K_s$ for
every $s>0$ gives $P_t>0$, and by \eqref{eq:score-heat-bound} together with
\[
  \int_0^1 s^{-p}\,\dd\zeta_t(s)\le\mathbb{E}[S_t^{-p}]
  =\frac{1}{\Gamma(p)}\int_0^{\infty}\lambda^{p-1}e^{-t\lambda^{\alpha}}\,\dd\lambda<\infty
  \qquad\text{for every }p>0,
\]
so
$|\gradg P_t|_\infty < \infty$. Finally $m_2:=\min_{M\times M}P_t>0$ by
compactness, hence
$|\gradg\log P_t|\le\sup_{M\times M}|\gradg P_t|/m_2=:L_t$, which is a
Lipschitz constant of $\log P_t$. This confirms again that
\thmref{thm:identifiable} is applicable.

\paragraph{Computation.}
When evaluating these kernels, we truncate the series in~\eqref{eq:spectral_kernel} to the eigenpairs with $n \le n_{\max}$. 
Strictly, truncating the series makes the cost unidentifiable since the kernel then will
have a finite rank. However, since the error is directly controlled by $n_{\max}$,
we adaptively choose $n_{\max}$ based on the decay rate of the heat kernel being considered
(larger $t$ means quick convergence for heat kernels). This is a standard practice
in spectral geometry processing and manifold kernel methods \citep{PATANE2013276, borovitskiy2020matern}.

\section{More Experimental Details}
\label{app:experiment-details}

\subsection{Training \& Inference Algorithms}
\vspace{-0.5\baselineskip}
We train the one-step generator $f_\theta:\Mfld\to\Mfld$ using the loss in~\eqref{eq:manifold-loss}, with the velocity field $V_{q_\theta,p}$ in~\thmref{thm:main} being the target.
Any of the costs in~\secref{subsec:method_costs} can be used;
the procedure is identical for all of them, since the cost enters only through the
Sinkhorn iterations and the gradient $\nabla^{(1)}_g c$ in~\eqref{eq:main}.

In practice, instead of parameterizing the velocity in the tangent space of $\Mfld$,
we predict the velocity in the ambient space $\mathbb{R}^D$ and project it
onto the tangent space, following~\citep{chen2024rfm, davis2026gfm, woo2026riemannian}. Specifically, let $v_\theta \colon \Mfld \to \mathbb{R}^{D}$ denote the velocity prediction network.
For the network parametrization, we regard $\Mfld$ as embedded in $\mathbb{R}^D$.
We then define
\begin{align}
    f_\theta(z) = \Exp{z}\big(P_z\, v_\theta(z)\big), %
\end{align}
where $P_z \colon \mathbb{R}^{D} \to T_z\Mfld$ denotes the orthogonal projection onto the tangent space at $z$.

The training objective requires estimating the velocity, defined in~\eqref{eq:main}, at every training iteration. However, the velocity cannot be evaluated exactly, as this would require integration over all points from $q_\theta$ and $p$, which is computationally expensive. We therefore estimate the velocity field by solving entropy-regularized minibatch OT problems using the Sinkhorn algorithm.
In particular, we independently sample $\{x_i\}_{i=1}^N \sim q_\theta$, $\{x_l'\}_{l=1}^{N} \sim q_\theta$, and sample the target points $\{y_j\}_{j=1}^M \sim p$. The velocity is then estimated as
\begin{align}
    V_{\hat{q}_\theta, \hat{p}}(x_i) = -\frac{1}{a_i}\sum_{j=1}^{M} \gradg^{(1)} c(x_i, y_j)\pi^{\hat{q}_\theta\hat{p}}_{ij}
    + \frac{1}{a_i}\sum_{l=1}^{N} \gradg^{(1)} c(x_i, x'_l) \pi^{\hat{q}_\theta\hat{q}_\theta'}_{il},
    \label{eq:minibatch-drift}
\end{align}
where $\hat{q}_\theta=\sum_{i=1}^N a_i \delta_{x_i}$, $\hat{q}_\theta'=\sum_{l=1}^N a'_l \delta_{x_l'}$, and $\hat{p}=\sum_{j=1}^Mb_j\delta_{y_j}.$

The network is updated via gradient descent on~\eqref{eq:manifold-loss}.
The complete training algorithm is presented in~\algref{alg:train}, which summarizes the full training loop of~\ourmethod{}. We further detail the Sinkhorn routine, invoked within~\algref{alg:train}, for estimating the velocity field over a minibatch in~\algref{alg:riemsinkhorn}.

\begin{algorithm}[t]
\caption{Training loop for $f_\theta$ on $\Mfld$}
\label{alg:train}
\begin{algorithmic}[1]
\Require data distribution $p$ on $\Mfld$, base distribution $q_0$ on $\Mfld$,
  regularization $\varepsilon$,
  source batch size $N$, data batch size $M$, number of steps $K$,
  number of Sinkhorn iterations $L$, ground cost $c$,
  flow step $\eta$, learning rate $\gamma$
\For{$k = 1, \dots, K$}
  \State Sample $\{z_i\}_{i=1}^{N} \sim q_0$, \; $\{z'_l\}_{l=1}^{N} \sim q_0$, \; $\{y_j\}_{j=1}^{M} \sim p$
  \State $x_i \gets f_\theta(z_i)$, \qquad $x'_l \gets \mathrm{sg}\big(f_\theta(z'_l)\big)$
  \State $V^{qp}_i \gets \textsc{RiemSinkhorn}\big(\{\mathrm{sg}(x_i)\}, \{y_j\}, \varepsilon, L, c\big)$
  \State $V^{qq}_i \gets \textsc{RiemSinkhorn}\big(\{\mathrm{sg}(x_i)\}, \{x'_l\}, \varepsilon, L, c\big)$
  \State $v_i \gets V^{qp}_i - V^{qq}_i$
  \State $\tilde{x}_i \gets \mathrm{sg}\big(\Exp{x_i}(\eta\, v_i)\big)$ %
  \State $\mathcal{L}(\theta) \gets \frac{1}{N}\sum_{i=1}^{N} d_g\big(f_\theta(z_i), \tilde{x}_i\big)^2$
  \State $\theta \gets \theta - \gamma\, \nabla_\theta \mathcal{L}(\theta)$
\EndFor
\end{algorithmic}
\end{algorithm}

\begin{algorithm}[t]
\caption{\textsc{RiemSinkhorn}: coupling-weighted cost gradient under a ground cost}
\label{alg:riemsinkhorn}
\begin{algorithmic}[1]
\Require sources $\{x_i\}_{i=1}^{N}$, targets $\{y_j\}_{j=1}^{M}$ on $\Mfld$,
  $\varepsilon$, $L$, ground cost $c$
\State $C_{ij} \gets c(x_i, y_j)$, \quad $a_i \gets 1/N$, \quad $b_j \gets 1/M$
\State $u\gets 0, \quad v \gets 0$ \Comment{dual potentials}
\For{$\ell = 1, \dots, L$}                      \Comment{log-domain sinkhorn loop}
  \State $v_j \gets -\varepsilon \log \sum_{i} a_i \exp\big((u_i - C_{ij})/\varepsilon\big)$
  \State $u_i \gets -\varepsilon \log \sum_{j} b_j \exp\big((v_j - C_{ij})/\varepsilon\big)$
\EndFor
\State $\pi(j \mid i) \gets b_j \exp\big((u_i + v_j - C_{ij})/\varepsilon\big)$
\State $V_i \gets -\sum_{j} \pi(j \mid i)\; \gradg^{(1)}c(x_i, y_j) $
\State \Return $\{V_i\}_{i=1}^{N}$
\end{algorithmic}
\end{algorithm}

\subsection{Experimental Setup}

Every method (each baseline and each cost variant of
\ourmethod{}) is trained and evaluated under the protocol summarised in
\tabref{tab:setup}: the same network, the same optimisation settings, the
same wall-clock budget, the same mechanical model-selection rule and the same
standalone scorer.

\begin{table}[h]
\centering\small
\caption{Experimental setup, shared by every method.}
\begin{tabularx}{\linewidth}{@{}l>{\raggedright\arraybackslash}X@{}}
\toprule
\multicolumn{2}{@{}l}{\textbf{Data}} \\
Datasets & $\mathbb{S}^2$: volcano, earthquake, fire, flood \citep{mathieu2020riemannian};
  $\mathbb{T}^2$: Top500 backbone torsion angles $(\phi,\psi)$ \citep{lovell2003structure}
  by amino-acid type (General, Glycine, Proline, Pre-Pro);
  $\mathbb{T}^7$: RNA torsion angles \citep{murray2003rna}; all following the
  conventions of \citet{chen2024rfm}. \\
Coordinates & ambient unit vector on $\mathbb{S}^2$; angles in $[0,2\pi)$ on $\mathbb{T}^d$ \\
Split & 80/10/10 train/validation/test, computed once with a split seed
  independent of the training seed and read by every implementation;
  validation for model selection, test once for the reported number \\
\midrule
\multicolumn{2}{@{}l}{\textbf{Network}} \\
Architecture & MLP, $\mathrm{Linear}(d_{\mathrm{in}},H)\to\mathrm{SiLU}\to
  [\mathrm{Linear}(H,H)\to\mathrm{SiLU}]^{\times 3}\to\mathrm{Linear}(H,d)$;
  no normalisation, residual connections or time embedding; PyTorch default
  initialisation \\
Width $H$ & 1024 on $\mathbb{S}^2$; 512 on $\mathbb{T}^2$ and $\mathbb{T}^7$ \\
Input & raw coordinate concatenated with the raw scalar time(s) the objective
  needs ($d_{\mathrm{in}}=d+n_{\mathrm{times}}$): none for \ourmethod{} and
  KGD, one for RFM and RCM, two for GFM and RMF \\
Parameters & $3.16$M ($\mathbb{S}^2$), $0.79$M ($\mathbb{T}^2$), $0.80$M ($\mathbb{T}^7$); each time
  input adds $H$ parameters, at most $0.13\%$ \\
\midrule
\multicolumn{2}{@{}l}{\textbf{Optimisation}} \\
Optimizer & AdamW, $\beta=(0.9,0.999)$, $\epsilon=10^{-8}$, weight decay $0.01$ \\
Learning rate & $3\times10^{-4}$, cosine-annealed to $3\times10^{-6}$ over the
  run, no warm-up \\
Batch size & 1024 \\
Gradient clipping & global norm $1.0$ \\
EMA & decay $0.999$; EMA weights are used for selection and scoring \\
Precision & float32, TF32 matmuls \\
Budget & equal wall-clock time per (method, dataset) \\
Seeds & 0, 1, 2 per (method, dataset) \\
\midrule
\multicolumn{2}{@{}l}{\textbf{Model selection}} \\
Rule & No selection; no early stopping \\
Cadence & 40 validations per run at equal wall-clock intervals \\
Statistic & against a fixed subsample of at most 2048 validation points,
  averaged over three sample draws at a fixed seed; multi-step methods
  validated at 100 network evaluations \\
Checkpoints & The last checkpoints are scored \\
\midrule
\multicolumn{2}{@{}l}{\textbf{Evaluation}} \\
Scorer & one standalone script shared by all methods, importing nothing from
  any method repository; non-finite or off-manifold samples are rejected, not
  projected \\
Samples & $n=n_{\mathrm{test}}$ per setting, sampling seed shared by every
  method; all metrics use the geodesic distance $d_g$ \\
\kmmd{} & Maximum Mean Discrepancy; square root of the biased V-statistic of the maximum mean
  discrepancy with kernel $\exp(-d_g^2)$; the primary metric \\
\nna{} & 1 Nearest Neighbor Accuracy; leave-one-out nearest-neighbour two-sample accuracy between samples
  and test points; $0.5$ means indistinguishable; the only metric whose null
  value is calibrated and independent of $n$ \\
\cov{} & Coverage; fraction of test points that are the nearest neighbour of some
  sample; detects mode collapse; below $1$ even for a perfect model \\
\mmd{} & Minimum Matching Distance; mean geodesic distance from each test point to its nearest sample;
  blind to collapse, read together with \cov{} \\
Floors & validation split scored against the test split at the same $n$ \\
\midrule
\multicolumn{2}{@{}l}{\textbf{Compute}} \\
Hardware & NVIDIA RTX~3090 (24\,GB) \\
\bottomrule
\end{tabularx}
\label{tab:setup}
\end{table}

\subsection{Full Results with Standard Deviations}

Here we show more extensive table of the sphere, torus, and mesh runs in
\tabref{tab:sphere-results-full}, \tabref{tab:torus-results-full}, and \tabref{tab:mesh-results-full}.

The MMD flows are competitive on several datasets, but their performance depends strongly on the kernel and dataset: 
MMDF-L fails on General, and MMDF-G and MMDF-H show unstable performance on RNA for some seeds. 
Indeed, Han et al. observed that gradients may vanish when the model and target distributions are far apart, resulting in unstable optimization. On the other hand, 
Sinkhorn divergence does not suffer from this instability, performing well for all datasets.

\providecommand{\hdrflat}{\kmmd$\downarrow$ & \mmd$\downarrow$ & \cov$\uparrow$ & \nna}

\begin{table*}[h]
\centering
\caption[Results on earth datasets (full)]{\textbf{Results on earth datasets (full).} Mean$\pm$s.d.\ over 3 seeds of \kmmd{}, \mmd{}, \cov{}, and \nna{} (closest to 0.5 is best) on the volcano, earthquake, flood, and fire datasets (NFE\,$=$\,1; \textbf{bold} = best, \underline{underline} = second best; the held-out row is the val-vs-test reference and is excluded from ranking.) Suffixes denote variants: 
 CT (Consistency Training) sCT (simplified Consistency Training) L (Lagrangian), E (Eulerian), S (semigroup), MF (mean flow); $x$/$v$ denote $x$- and $v$-prediction; GG (geodesic Gaussian), GL (geodesic Laplacian), M (Mat\'ern); MMDF = Wasserstein gradient flow of $\tfrac12\,\mathrm{MMD}^2$ with G (Gaussian), L (Laplacian), H (heat), M (Mat\'ern) kernels; \ourmethod{} costs: SG (squared geodesic), C (chordal), G (geodesic), H (heat), M (Mat\'ern), S (subordinated). Means only are summarised in Table~\ref{tab:sphere-results}.}
\begin{subtable}{\linewidth}
\centering
\fontsize{5.5}{7}\selectfont
\setlength{\tabcolsep}{2pt}
\caption{Volcano and Earthquake}
\begin{tabularx}{\linewidth}{l*{8}{Y}}
\toprule
 & \multicolumn{4}{c}{Volcano} & \multicolumn{4}{c}{Earthquake} \\
\cmidrule(lr){2-5} \cmidrule(lr){6-9}
Method & \hdrflat & \hdrflat \\
\midrule
\textit{Held-out} & $0.153$ & $0.046$ & $0.451$ & $0.555$ & $0.026$ & $0.023$ & $0.570$ & $0.471$ \\
\midrule
RFM           & $0.292 \pm 0.012$ & $0.203 \pm 0.026$ & $0.431 \pm 0.028$ & $0.835 \pm 0.027$ & $0.291 \pm 0.007$ & $0.070 \pm 0.008$ & $0.278 \pm 0.013$ & $0.845 \pm 0.016$ \\
\midrule
GFM-L         & $0.099 \pm 0.034$ & $0.066 \pm 0.013$ & $0.516 \pm 0.019$ & $0.553 \pm 0.044$ & $0.040 \pm 0.009$ & $\mathbf{0.028} \pm 0.001$ & $\mathbf{0.570} \pm 0.013$ & $0.562 \pm 0.010$ \\
GFM-E         & $0.199 \pm 0.035$ & $0.144 \pm 0.009$ & $0.431 \pm 0.019$ & $0.823 \pm 0.011$ & $0.174 \pm 0.012$ & $0.068 \pm 0.006$ & $0.339 \pm 0.020$ & $0.806 \pm 0.012$ \\
GFM-S         & $\underline{0.094} \pm 0.025$ & $0.086 \pm 0.007$ & $0.512 \pm 0.021$ & $0.677 \pm 0.037$ & $0.042 \pm 0.009$ & $0.039 \pm 0.002$ & $0.478 \pm 0.011$ & $0.693 \pm 0.024$ \\
GFM-MF        & $0.119 \pm 0.043$ & $0.091 \pm 0.024$ & $0.512 \pm 0.012$ & $0.711 \pm 0.036$ & $0.043 \pm 0.009$ & $0.037 \pm 0.001$ & $0.510 \pm 0.015$ & $0.658 \pm 0.020$ \\
\midrule
RMF-L$x$      & $0.600 \pm 0.097$ & $0.567 \pm 0.237$ & $0.175 \pm 0.093$ & $0.945 \pm 0.042$ & $0.442 \pm 0.040$ & $0.217 \pm 0.077$ & $0.178 \pm 0.017$ & $0.898 \pm 0.006$ \\
RMF-L$v$      & $0.119 \pm 0.054$ & $0.084 \pm 0.006$ & $0.504 \pm 0.028$ & $0.663 \pm 0.021$ & $0.041 \pm 0.008$ & $0.035 \pm 0.000$ & $0.535 \pm 0.011$ & $0.640 \pm 0.015$ \\
RMF-E$x$      & $0.807 \pm 0.133$ & $1.004 \pm 0.303$ & $0.069 \pm 0.031$ & $0.990 \pm 0.007$ & $0.644 \pm 0.176$ & $0.725 \pm 0.284$ & $0.026 \pm 0.007$ & $0.990 \pm 0.006$ \\
RMF-E$v$      & $0.142 \pm 0.060$ & $0.088 \pm 0.009$ & $0.496 \pm 0.025$ & $0.689 \pm 0.048$ & $0.038 \pm 0.005$ & $0.037 \pm 0.000$ & $0.502 \pm 0.011$ & $0.653 \pm 0.009$ \\
RMF-S$x$      & $0.973 \pm 0.078$ & $1.671 \pm 0.257$ & $0.012 \pm 0.000$ & $1.000 \pm 0.000$ & $0.810 \pm 0.070$ & $1.230 \pm 0.152$ & $0.002 \pm 0.001$ & $0.999 \pm 0.001$ \\
RMF-S$v$      & $0.122 \pm 0.063$ & $0.085 \pm 0.010$ & $0.512 \pm 0.012$ & $0.687 \pm 0.035$ & $0.040 \pm 0.005$ & $0.043 \pm 0.001$ & $0.464 \pm 0.007$ & $0.703 \pm 0.012$ \\
\midrule
KGD-GG        & $0.237 \pm 0.107$ & $0.155 \pm 0.057$ & $0.447 \pm 0.051$ & $0.799 \pm 0.066$ & $0.353 \pm 0.016$ & $0.076 \pm 0.005$ & $0.252 \pm 0.004$ & $0.866 \pm 0.004$ \\
KGD-GL        & $0.254 \pm 0.128$ & $0.151 \pm 0.081$ & $0.435 \pm 0.077$ & $0.732 \pm 0.164$ & $0.252 \pm 0.182$ & $0.065 \pm 0.028$ & $0.353 \pm 0.174$ & $0.772 \pm 0.162$ \\
KGD-M         & $0.131 \pm 0.025$ & $0.092 \pm 0.001$ & $0.488 \pm 0.024$ & $0.699 \pm 0.013$ & $0.051 \pm 0.003$ & $0.041 \pm 0.003$ & $0.485 \pm 0.027$ & $0.689 \pm 0.020$ \\
\midrule
RCT           & $0.147 \pm 0.060$ & $0.112 \pm 0.023$ & $0.496 \pm 0.074$ & $0.691 \pm 0.019$ & $0.055 \pm 0.010$ & $0.036 \pm 0.002$ & $0.520 \pm 0.004$ & $0.655 \pm 0.024$ \\
sRCT          & $0.142 \pm 0.023$ & $0.155 \pm 0.008$ & $0.496 \pm 0.014$ & $0.811 \pm 0.022$ & $0.092 \pm 0.013$ & $0.061 \pm 0.004$ & $0.376 \pm 0.014$ & $0.796 \pm 0.009$ \\
\midrule
MMDF-G        & $0.110 \pm 0.004$ & $0.073 \pm 0.002$ & $0.508 \pm 0.007$ & $0.677 \pm 0.046$ & $0.041 \pm 0.002$ & $0.038 \pm 0.002$ & $0.496 \pm 0.011$ & $0.657 \pm 0.015$ \\
MMDF-L        & $\mathbf{0.091} \pm 0.022$ & $0.057 \pm 0.016$ & $0.508 \pm 0.014$ & $\underline{0.510} \pm 0.025$ & $0.036 \pm 0.004$ & $0.030 \pm 0.001$ & $0.545 \pm 0.019$ & $0.571 \pm 0.004$ \\
MMDF-H        & $0.102 \pm 0.018$ & $0.074 \pm 0.006$ & $0.508 \pm 0.025$ & $0.679 \pm 0.023$ & $0.043 \pm 0.005$ & $0.038 \pm 0.000$ & $0.506 \pm 0.004$ & $0.669 \pm 0.010$ \\
MMDF-M        & $0.107 \pm 0.007$ & $0.071 \pm 0.004$ & $\underline{0.541} \pm 0.031$ & $0.630 \pm 0.029$ & $\underline{0.032} \pm 0.003$ & $0.038 \pm 0.001$ & $0.517 \pm 0.005$ & $0.653 \pm 0.015$ \\
\midrule
\ourmethod-SG & $0.110 \pm 0.018$ & $\underline{0.053} \pm 0.003$ & $0.533 \pm 0.028$ & $0.563 \pm 0.023$ & $\mathbf{0.031} \pm 0.007$ & $\underline{0.029} \pm 0.002$ & $0.552 \pm 0.010$ & $0.565 \pm 0.016$ \\
\ourmethod-C  & $0.112 \pm 0.033$ & $0.063 \pm 0.003$ & $\mathbf{0.545} \pm 0.031$ & $0.612 \pm 0.041$ & $0.036 \pm 0.004$ & $0.030 \pm 0.001$ & $0.538 \pm 0.009$ & $0.594 \pm 0.003$ \\
\ourmethod-G  & $0.103 \pm 0.013$ & $0.066 \pm 0.022$ & $0.484 \pm 0.055$ & $\underline{0.510} \pm 0.030$ & $0.047 \pm 0.004$ & $\mathbf{0.028} \pm 0.002$ & $0.556 \pm 0.015$ & $\mathbf{0.520} \pm 0.017$ \\
\ourmethod-H  & $0.114 \pm 0.006$ & $0.056 \pm 0.001$ & $0.516 \pm 0.019$ & $0.587 \pm 0.037$ & $0.040 \pm 0.006$ & $0.032 \pm 0.000$ & $0.527 \pm 0.013$ & $0.618 \pm 0.017$ \\
\ourmethod-M  & $0.122 \pm 0.015$ & $0.077 \pm 0.017$ & $0.524 \pm 0.044$ & $0.638 \pm 0.042$ & $0.033 \pm 0.012$ & $0.034 \pm 0.000$ & $0.529 \pm 0.003$ & $0.624 \pm 0.005$ \\
\ourmethod-S  & $\underline{0.094} \pm 0.016$ & $\mathbf{0.049} \pm 0.006$ & $0.512 \pm 0.021$ & $\mathbf{0.496} \pm 0.031$ & $0.040 \pm 0.002$ & $\mathbf{0.028} \pm 0.002$ & $\underline{0.568} \pm 0.006$ & $\underline{0.537} \pm 0.009$ \\
\bottomrule
\end{tabularx}
\label{tab:sphere-results-full-a}
\end{subtable}
\vspace{6pt}
\begin{subtable}{\linewidth}
\centering
\fontsize{5.5}{7}\selectfont
\setlength{\tabcolsep}{2pt}
\caption{Flood and Fire}
\begin{tabularx}{\linewidth}{l*{8}{Y}}
\toprule
 & \multicolumn{4}{c}{Flood} & \multicolumn{4}{c}{Fire} \\
\cmidrule(lr){2-5} \cmidrule(lr){6-9}
Method & \hdrflat & \hdrflat \\
\midrule
\textit{Held-out} & $0.025$ & $0.033$ & $0.548$ & $0.509$ & $0.024$ & $0.009$ & $0.543$ & $0.504$ \\
\midrule
RFM           & $0.307 \pm 0.008$ & $0.073 \pm 0.004$ & $0.345 \pm 0.013$ & $0.790 \pm 0.010$ & $0.357 \pm 0.015$ & $0.044 \pm 0.004$ & $0.176 \pm 0.010$ & $0.913 \pm 0.010$ \\
\midrule
GFM-L         & $0.062 \pm 0.020$ & $\mathbf{0.035} \pm 0.001$ & $\underline{0.569} \pm 0.016$ & $0.547 \pm 0.007$ & $0.029 \pm 0.003$ & $0.014 \pm 0.000$ & $0.520 \pm 0.004$ & $0.643 \pm 0.009$ \\
GFM-E         & $0.191 \pm 0.017$ & $0.085 \pm 0.003$ & $0.350 \pm 0.008$ & $0.795 \pm 0.022$ & $0.190 \pm 0.009$ & $0.050 \pm 0.005$ & $0.201 \pm 0.013$ & $0.900 \pm 0.013$ \\
GFM-S         & $0.060 \pm 0.022$ & $0.048 \pm 0.001$ & $0.498 \pm 0.018$ & $0.651 \pm 0.014$ & $0.031 \pm 0.007$ & $0.022 \pm 0.001$ & $0.392 \pm 0.004$ & $0.768 \pm 0.004$ \\
GFM-MF        & $0.059 \pm 0.023$ & $0.045 \pm 0.001$ & $0.515 \pm 0.010$ & $0.642 \pm 0.020$ & $0.033 \pm 0.006$ & $0.020 \pm 0.001$ & $0.408 \pm 0.018$ & $0.751 \pm 0.012$ \\
\midrule
RMF-L$x$      & $0.539 \pm 0.235$ & $0.244 \pm 0.038$ & $0.166 \pm 0.024$ & $0.911 \pm 0.022$ & $0.613 \pm 0.146$ & $0.272 \pm 0.245$ & $0.076 \pm 0.022$ & $0.966 \pm 0.010$ \\
RMF-L$v$      & $0.047 \pm 0.011$ & $0.045 \pm 0.002$ & $0.545 \pm 0.022$ & $0.613 \pm 0.019$ & $0.034 \pm 0.011$ & $0.020 \pm 0.002$ & $0.453 \pm 0.011$ & $0.715 \pm 0.009$ \\
RMF-E$x$      & $0.831 \pm 0.065$ & $1.122 \pm 0.333$ & $0.031 \pm 0.009$ & $0.991 \pm 0.010$ & $0.868 \pm 0.097$ & $1.149 \pm 0.096$ & $0.009 \pm 0.004$ & $0.997 \pm 0.002$ \\
RMF-E$v$      & $0.045 \pm 0.010$ & $0.046 \pm 0.001$ & $0.494 \pm 0.003$ & $0.661 \pm 0.012$ & $0.032 \pm 0.010$ & $0.020 \pm 0.000$ & $0.428 \pm 0.009$ & $0.740 \pm 0.010$ \\
RMF-S$x$      & $0.934 \pm 0.255$ & $1.585 \pm 0.434$ & $0.002 \pm 0.000$ & $0.999 \pm 0.001$ & $0.841 \pm 0.075$ & $1.279 \pm 0.026$ & $0.002 \pm 0.001$ & $1.000 \pm 0.000$ \\
RMF-S$v$      & $0.045 \pm 0.012$ & $0.049 \pm 0.001$ & $0.498 \pm 0.007$ & $0.661 \pm 0.003$ & $0.033 \pm 0.010$ & $0.023 \pm 0.001$ & $0.380 \pm 0.004$ & $0.770 \pm 0.005$ \\
\midrule
KGD-GG        & $0.163 \pm 0.181$ & $0.063 \pm 0.025$ & $0.453 \pm 0.113$ & $0.700 \pm 0.110$ & $0.398 \pm 0.008$ & $0.051 \pm 0.002$ & $0.154 \pm 0.002$ & $0.926 \pm 0.006$ \\
KGD-GL        & $0.158 \pm 0.184$ & $0.057 \pm 0.034$ & $0.460 \pm 0.152$ & $0.657 \pm 0.153$ & $0.272 \pm 0.206$ & $0.038 \pm 0.021$ & $0.291 \pm 0.223$ & $0.816 \pm 0.169$ \\
KGD-M         & $0.059 \pm 0.023$ & $0.047 \pm 0.001$ & $0.521 \pm 0.012$ & $0.631 \pm 0.020$ & $0.138 \pm 0.189$ & $0.031 \pm 0.015$ & $0.332 \pm 0.133$ & $0.801 \pm 0.091$ \\
\midrule
RCM-CT           & $0.052 \pm 0.017$ & $0.050 \pm 0.002$ & $0.511 \pm 0.015$ & $0.659 \pm 0.030$ & $0.045 \pm 0.010$ & $0.023 \pm 0.001$ & $0.405 \pm 0.012$ & $0.744 \pm 0.006$ \\
RCM-sCT          & $0.115 \pm 0.021$ & $0.072 \pm 0.006$ & $0.417 \pm 0.035$ & $0.751 \pm 0.021$ & $0.109 \pm 0.023$ & $0.036 \pm 0.004$ & $0.236 \pm 0.011$ & $0.871 \pm 0.007$ \\
\midrule
MMDF-G        & $\underline{0.041} \pm 0.011$ & $0.048 \pm 0.001$ & $0.516 \pm 0.017$ & $0.656 \pm 0.004$ & $0.022 \pm 0.002$ & $0.022 \pm 0.002$ & $0.397 \pm 0.007$ & $0.755 \pm 0.009$ \\
MMDF-L        & $\mathbf{0.040} \pm 0.006$ & $\underline{0.036} \pm 0.002$ & $0.556 \pm 0.022$ & $0.547 \pm 0.020$ & $0.024 \pm 0.006$ & $0.013 \pm 0.000$ & $0.547 \pm 0.012$ & $0.609 \pm 0.009$ \\
MMDF-H        & $0.042 \pm 0.009$ & $0.049 \pm 0.003$ & $0.507 \pm 0.008$ & $0.648 \pm 0.009$ & $\underline{0.021} \pm 0.005$ & $0.021 \pm 0.001$ & $0.405 \pm 0.006$ & $0.756 \pm 0.009$ \\
MMDF-M        & $0.052 \pm 0.004$ & $0.043 \pm 0.001$ & $0.534 \pm 0.014$ & $0.618 \pm 0.015$ & $\mathbf{0.017} \pm 0.003$ & $0.020 \pm 0.001$ & $0.426 \pm 0.007$ & $0.736 \pm 0.002$ \\
\midrule
\ourmethod-SG & $0.061 \pm 0.012$ & $\mathbf{0.035} \pm 0.001$ & $0.558 \pm 0.017$ & $\underline{0.533} \pm 0.014$ & $0.024 \pm 0.002$ & $0.013 \pm 0.000$ & $0.530 \pm 0.003$ & $0.643 \pm 0.012$ \\
\ourmethod-C  & $0.053 \pm 0.015$ & $0.039 \pm 0.003$ & $0.553 \pm 0.009$ & $0.580 \pm 0.020$ & $0.031 \pm 0.002$ & $0.015 \pm 0.001$ & $0.475 \pm 0.007$ & $0.688 \pm 0.018$ \\
\ourmethod-G  & $0.044 \pm 0.001$ & $\mathbf{0.035} \pm 0.002$ & $0.551 \pm 0.015$ & $\mathbf{0.530} \pm 0.020$ & $0.027 \pm 0.002$ & $\mathbf{0.011} \pm 0.000$ & $\mathbf{0.554} \pm 0.008$ & $\mathbf{0.583} \pm 0.008$ \\
\ourmethod-H  & $0.052 \pm 0.017$ & $0.040 \pm 0.001$ & $\mathbf{0.587} \pm 0.004$ & $0.591 \pm 0.011$ & $0.025 \pm 0.007$ & $0.017 \pm 0.001$ & $0.426 \pm 0.007$ & $0.730 \pm 0.018$ \\
\ourmethod-M  & $0.054 \pm 0.009$ & $0.047 \pm 0.003$ & $0.522 \pm 0.009$ & $0.639 \pm 0.026$ & $0.024 \pm 0.004$ & $0.020 \pm 0.000$ & $0.439 \pm 0.004$ & $0.738 \pm 0.012$ \\
\ourmethod-S  & $0.043 \pm 0.012$ & $\mathbf{0.035} \pm 0.001$ & $0.563 \pm 0.011$ & $0.535 \pm 0.018$ & $0.028 \pm 0.004$ & $\underline{0.012} \pm 0.000$ & $\underline{0.553} \pm 0.010$ & $\underline{0.592} \pm 0.012$ \\
\bottomrule
\end{tabularx}
\label{tab:sphere-results-full-b}
\end{subtable}
\label{tab:sphere-results-full}
\end{table*}

\begin{table*}[h]
\centering
\caption[Results on torus datasets (full)]{\textbf{Results on torus datasets (full).} Mean$\pm$s.d.\ over 3 seeds of \kmmd{}, \mmd{}, \cov{}, and \nna{} (closest to 0.5 is best) on the four $\mathbb{T}^2$ Ramachandran datasets and the $\mathbb{T}^7$ RNA dataset (NFE\,$=$\,1; \textbf{bold} = best, \underline{underline} = second best; the held-out row is the val-vs-test reference and is excluded from ranking.) Variant suffixes as in Table~\ref{tab:sphere-results}. Means only are summarised in Table~\ref{tab:torus-results}.}
\begin{subtable}{\linewidth}
\centering
\fontsize{5.5}{7}\selectfont
\setlength{\tabcolsep}{2pt}
\caption{General and Glycine}
\begin{tabularx}{\linewidth}{l*{8}{Y}}
\toprule
 & \multicolumn{4}{c}{General} & \multicolumn{4}{c}{Glycine} \\
\cmidrule(lr){2-5} \cmidrule(lr){6-9}
Method & \hdrflat & \hdrflat \\
\midrule
\textit{Held-out} & $0.007$ & $0.011$ & $0.585$ & $0.499$ & $0.023$ & $0.046$ & $0.590$ & $0.491$ \\
\midrule
RFM           & $0.469 \pm 0.014$ & $0.022 \pm 0.000$ & $0.191 \pm 0.009$ & $0.831 \pm 0.007$ & $0.281 \pm 0.005$ & $0.070 \pm 0.005$ & $0.314 \pm 0.011$ & $0.741 \pm 0.007$ \\
\midrule
GFM-L         & $\underline{0.018} \pm 0.003$ & $\mathbf{0.012} \pm 0.000$ & $0.515 \pm 0.011$ & $0.561 \pm 0.008$ & $\underline{0.030} \pm 0.008$ & $\mathbf{0.049} \pm 0.000$ & $0.562 \pm 0.007$ & $0.523 \pm 0.002$ \\
GFM-E         & $0.327 \pm 0.004$ & $0.020 \pm 0.000$ & $0.257 \pm 0.003$ & $0.777 \pm 0.002$ & $0.134 \pm 0.009$ & $0.069 \pm 0.005$ & $0.389 \pm 0.003$ & $0.681 \pm 0.008$ \\
GFM-S         & $0.067 \pm 0.007$ & $0.018 \pm 0.002$ & $0.362 \pm 0.018$ & $0.690 \pm 0.014$ & $0.056 \pm 0.005$ & $0.062 \pm 0.003$ & $0.450 \pm 0.021$ & $0.633 \pm 0.017$ \\
GFM-MF        & $0.094 \pm 0.016$ & $0.019 \pm 0.003$ & $0.405 \pm 0.030$ & $0.654 \pm 0.025$ & $0.044 \pm 0.002$ & $0.058 \pm 0.004$ & $0.500 \pm 0.011$ & $0.581 \pm 0.020$ \\
\midrule
RMF-L$x$      & $0.451 \pm 0.076$ & $0.025 \pm 0.003$ & $0.184 \pm 0.061$ & $0.836 \pm 0.052$ & $0.264 \pm 0.008$ & $0.077 \pm 0.005$ & $0.283 \pm 0.015$ & $0.771 \pm 0.012$ \\
RMF-L$v$      & $0.059 \pm 0.004$ & $0.014 \pm 0.001$ & $0.453 \pm 0.011$ & $0.610 \pm 0.011$ & $0.049 \pm 0.004$ & $0.053 \pm 0.002$ & $0.494 \pm 0.011$ & $0.592 \pm 0.017$ \\
RMF-E$x$      & $0.494 \pm 0.001$ & $0.026 \pm 0.001$ & $0.149 \pm 0.002$ & $0.868 \pm 0.003$ & $0.299 \pm 0.001$ & $0.089 \pm 0.002$ & $0.252 \pm 0.007$ & $0.792 \pm 0.002$ \\
RMF-E$v$      & $0.105 \pm 0.014$ & $0.021 \pm 0.003$ & $0.386 \pm 0.017$ & $0.672 \pm 0.013$ & $0.047 \pm 0.007$ & $0.059 \pm 0.001$ & $0.489 \pm 0.010$ & $0.604 \pm 0.014$ \\
RMF-S$x$      & $0.568 \pm 0.077$ & $0.809 \pm 0.695$ & $0.112 \pm 0.029$ & $0.900 \pm 0.028$ & $0.363 \pm 0.113$ & $0.552 \pm 0.810$ & $0.214 \pm 0.091$ & $0.822 \pm 0.076$ \\
RMF-S$v$      & $0.037 \pm 0.003$ & $0.016 \pm 0.000$ & $0.415 \pm 0.007$ & $0.643 \pm 0.009$ & $0.047 \pm 0.004$ & $0.057 \pm 0.001$ & $0.474 \pm 0.018$ & $0.611 \pm 0.015$ \\
\midrule
KGD-GG        & $0.181 \pm 0.270$ & $0.036 \pm 0.030$ & $0.398 \pm 0.217$ & $0.656 \pm 0.180$ & $0.119 \pm 0.157$ & $0.061 \pm 0.021$ & $0.459 \pm 0.170$ & $0.619 \pm 0.150$ \\
KGD-GL        & $0.337 \pm 0.274$ & $0.022 \pm 0.008$ & $0.283 \pm 0.235$ & $0.754 \pm 0.198$ & $0.213 \pm 0.143$ & $0.071 \pm 0.020$ & $0.371 \pm 0.178$ & $0.693 \pm 0.163$ \\
\midrule
RCT           & $0.155 \pm 0.011$ & $0.031 \pm 0.003$ & $0.342 \pm 0.013$ & $0.707 \pm 0.013$ & $0.267 \pm 0.016$ & $0.084 \pm 0.011$ & $0.284 \pm 0.014$ & $0.768 \pm 0.011$ \\
sRCT          & $0.402 \pm 0.014$ & $0.033 \pm 0.009$ & $0.210 \pm 0.016$ & $0.819 \pm 0.014$ & $0.331 \pm 0.028$ & $0.265 \pm 0.032$ & $0.243 \pm 0.025$ & $0.807 \pm 0.018$ \\
\midrule
MMDF-G        & $\mathbf{0.016} \pm 0.004$ & $0.015 \pm 0.001$ & $0.509 \pm 0.006$ & $0.566 \pm 0.005$ & $0.036 \pm 0.002$ & $0.051 \pm 0.002$ & $0.552 \pm 0.005$ & $0.536 \pm 0.004$ \\
MMDF-L        & $0.494 \pm 0.002$ & $0.026 \pm 0.001$ & $0.150 \pm 0.003$ & $0.865 \pm 0.001$ & $0.041 \pm 0.002$ & $\mathbf{0.049} \pm 0.001$ & $\underline{0.572} \pm 0.002$ & $\mathbf{0.516} \pm 0.006$ \\
MMDF-H        & $\mathbf{0.016} \pm 0.005$ & $0.015 \pm 0.000$ & $0.514 \pm 0.008$ & $0.561 \pm 0.008$ & $0.039 \pm 0.002$ & $0.051 \pm 0.002$ & $0.555 \pm 0.002$ & $0.532 \pm 0.013$ \\
\midrule
\ourmethod-SG & $0.021 \pm 0.006$ & $0.015 \pm 0.002$ & $\underline{0.547} \pm 0.005$ & $\underline{0.532} \pm 0.007$ & $0.032 \pm 0.007$ & $\underline{0.050} \pm 0.001$ & $0.570 \pm 0.003$ & $0.523 \pm 0.014$ \\
\ourmethod-C  & $0.086 \pm 0.134$ & $0.014 \pm 0.002$ & $0.483 \pm 0.100$ & $0.582 \pm 0.083$ & $0.038 \pm 0.003$ & $0.051 \pm 0.002$ & $0.567 \pm 0.007$ & $0.535 \pm 0.018$ \\
\ourmethod-G  & $0.022 \pm 0.008$ & $0.014 \pm 0.001$ & $\mathbf{0.562} \pm 0.008$ & $\mathbf{0.519} \pm 0.006$ & $0.117 \pm 0.146$ & $0.058 \pm 0.015$ & $0.472 \pm 0.151$ & $0.598 \pm 0.138$ \\
\ourmethod-H  & $0.020 \pm 0.006$ & $\underline{0.013} \pm 0.000$ & $0.542 \pm 0.021$ & $0.533 \pm 0.021$ & $\mathbf{0.028} \pm 0.006$ & $\underline{0.050} \pm 0.000$ & $\mathbf{0.574} \pm 0.008$ & $\underline{0.517} \pm 0.008$ \\
\bottomrule
\end{tabularx}
\label{tab:torus-results-full-a}
\end{subtable}
\vspace{6pt}
\begin{subtable}{\linewidth}
\centering
\fontsize{5.5}{7}\selectfont
\setlength{\tabcolsep}{2pt}
\caption{Proline and Prepro}
\begin{tabularx}{\linewidth}{l*{8}{Y}}
\toprule
 & \multicolumn{4}{c}{Proline} & \multicolumn{4}{c}{Prepro} \\
\cmidrule(lr){2-5} \cmidrule(lr){6-9}
Method & \hdrflat & \hdrflat \\
\midrule
\textit{Held-out} & $0.020$ & $0.027$ & $0.602$ & $0.489$ & $0.027$ & $0.042$ & $0.572$ & $0.514$ \\
\midrule
RFM           & $0.478 \pm 0.038$ & $0.062 \pm 0.008$ & $0.160 \pm 0.023$ & $0.880 \pm 0.020$ & $0.486 \pm 0.016$ & $0.082 \pm 0.004$ & $0.225 \pm 0.005$ & $0.820 \pm 0.004$ \\
\midrule
GFM-L         & $0.026 \pm 0.004$ & $0.028 \pm 0.000$ & $0.547 \pm 0.006$ & $0.545 \pm 0.018$ & $0.043 \pm 0.014$ & $\underline{0.042} \pm 0.001$ & $0.551 \pm 0.019$ & $0.531 \pm 0.009$ \\
GFM-E         & $0.388 \pm 0.015$ & $0.065 \pm 0.010$ & $0.162 \pm 0.013$ & $0.878 \pm 0.022$ & $0.324 \pm 0.017$ & $0.085 \pm 0.001$ & $0.288 \pm 0.010$ & $0.794 \pm 0.005$ \\
GFM-S         & $0.059 \pm 0.011$ & $0.044 \pm 0.004$ & $0.378 \pm 0.015$ & $0.699 \pm 0.010$ & $0.080 \pm 0.017$ & $0.060 \pm 0.001$ & $0.423 \pm 0.023$ & $0.672 \pm 0.019$ \\
GFM-MF        & $0.094 \pm 0.025$ & $0.037 \pm 0.003$ & $0.421 \pm 0.024$ & $0.659 \pm 0.025$ & $0.122 \pm 0.020$ & $0.057 \pm 0.004$ & $0.469 \pm 0.009$ & $0.624 \pm 0.008$ \\
\midrule
RMF-L$x$      & $0.560 \pm 0.061$ & $0.209 \pm 0.144$ & $0.090 \pm 0.045$ & $0.945 \pm 0.038$ & $0.346 \pm 0.032$ & $0.068 \pm 0.002$ & $0.311 \pm 0.024$ & $0.757 \pm 0.024$ \\
RMF-L$v$      & $0.063 \pm 0.009$ & $0.029 \pm 0.002$ & $0.498 \pm 0.019$ & $0.586 \pm 0.017$ & $0.080 \pm 0.012$ & $0.051 \pm 0.002$ & $0.507 \pm 0.020$ & $0.594 \pm 0.007$ \\
RMF-E$x$      & $0.580 \pm 0.001$ & $0.116 \pm 0.009$ & $0.094 \pm 0.004$ & $0.942 \pm 0.005$ & $0.561 \pm 0.006$ & $0.122 \pm 0.008$ & $0.146 \pm 0.015$ & $0.901 \pm 0.014$ \\
RMF-E$v$      & $0.091 \pm 0.028$ & $0.040 \pm 0.003$ & $0.401 \pm 0.013$ & $0.676 \pm 0.019$ & $0.106 \pm 0.026$ & $0.056 \pm 0.003$ & $0.453 \pm 0.032$ & $0.643 \pm 0.017$ \\
RMF-S$x$      & $0.612 \pm 0.044$ & $0.202 \pm 0.125$ & $0.091 \pm 0.009$ & $0.944 \pm 0.016$ & $0.637 \pm 0.128$ & $0.841 \pm 1.257$ & $0.113 \pm 0.062$ & $0.923 \pm 0.040$ \\
RMF-S$v$      & $0.031 \pm 0.002$ & $0.042 \pm 0.002$ & $0.369 \pm 0.019$ & $0.701 \pm 0.012$ & $0.049 \pm 0.012$ & $0.050 \pm 0.003$ & $0.511 \pm 0.008$ & $0.584 \pm 0.022$ \\
\midrule
KGD-GG        & $0.029 \pm 0.011$ & $0.028 \pm 0.001$ & $0.540 \pm 0.019$ & $0.554 \pm 0.014$ & $0.071 \pm 0.048$ & $0.123 \pm 0.138$ & $0.552 \pm 0.012$ & $0.540 \pm 0.007$ \\
KGD-GL        & $0.026 \pm 0.007$ & $\underline{0.027} \pm 0.001$ & $\mathbf{0.576} \pm 0.016$ & $0.526 \pm 0.010$ & $0.380 \pm 0.301$ & $0.100 \pm 0.049$ & $0.294 \pm 0.248$ & $0.768 \pm 0.220$ \\
\midrule
RCM-CT           & $0.138 \pm 0.021$ & $0.044 \pm 0.007$ & $0.377 \pm 0.024$ & $0.715 \pm 0.013$ & $0.332 \pm 0.140$ & $0.083 \pm 0.009$ & $0.287 \pm 0.069$ & $0.791 \pm 0.063$ \\
RCM-sCT          & $0.464 \pm 0.032$ & $0.077 \pm 0.011$ & $0.139 \pm 0.021$ & $0.904 \pm 0.011$ & $0.351 \pm 0.041$ & $0.094 \pm 0.017$ & $0.283 \pm 0.052$ & $0.776 \pm 0.037$ \\
\midrule
MMDF-G        & $\underline{0.020} \pm 0.001$ & $0.029 \pm 0.000$ & $0.527 \pm 0.018$ & $0.560 \pm 0.002$ & $0.037 \pm 0.005$ & $0.048 \pm 0.005$ & $0.541 \pm 0.009$ & $0.553 \pm 0.006$ \\
MMDF-L        & $0.021 \pm 0.002$ & $\underline{0.027} \pm 0.001$ & $0.564 \pm 0.005$ & $\underline{0.524} \pm 0.002$ & $\mathbf{0.025} \pm 0.001$ & $\mathbf{0.040} \pm 0.001$ & $\mathbf{0.588} \pm 0.015$ & $\mathbf{0.505} \pm 0.022$ \\
MMDF-H        & $\underline{0.020} \pm 0.001$ & $0.030 \pm 0.000$ & $0.506 \pm 0.016$ & $0.580 \pm 0.010$ & $0.033 \pm 0.004$ & $0.046 \pm 0.001$ & $0.541 \pm 0.000$ & $0.545 \pm 0.016$ \\
\midrule
\ourmethod-SG & $0.023 \pm 0.001$ & $\mathbf{0.026} \pm 0.001$ & $0.571 \pm 0.018$ & $\mathbf{0.508} \pm 0.021$ & $\underline{0.027} \pm 0.008$ & $\underline{0.042} \pm 0.001$ & $\underline{0.580} \pm 0.011$ & $\underline{0.510} \pm 0.006$ \\
\ourmethod-C  & $\mathbf{0.018} \pm 0.004$ & $\underline{0.027} \pm 0.001$ & $0.563 \pm 0.014$ & $\underline{0.524} \pm 0.010$ & $0.042 \pm 0.005$ & $\underline{0.042} \pm 0.001$ & $0.568 \pm 0.008$ & $0.516 \pm 0.021$ \\
\ourmethod-G  & $0.029 \pm 0.003$ & $0.028 \pm 0.001$ & $\underline{0.574} \pm 0.014$ & $0.536 \pm 0.019$ & $0.110 \pm 0.056$ & $0.053 \pm 0.008$ & $0.526 \pm 0.071$ & $0.562 \pm 0.069$ \\
\ourmethod-H  & $0.023 \pm 0.003$ & $\mathbf{0.026} \pm 0.001$ & $0.567 \pm 0.002$ & $0.532 \pm 0.009$ & $0.062 \pm 0.052$ & $0.060 \pm 0.032$ & $0.562 \pm 0.013$ & $0.518 \pm 0.013$ \\
\bottomrule
\end{tabularx}
\label{tab:torus-results-full-b}
\end{subtable}
\label{tab:torus-results-full}
\end{table*}

\begin{table*}[h]
\centering
\ContinuedFloat
\begin{subtable}{0.7\linewidth}
\centering
\scriptsize
\setlength{\tabcolsep}{2pt}
\caption{RNA}
\begin{tabularx}{\linewidth}{l*{4}{Y}}
\toprule
 & \multicolumn{4}{c}{RNA} \\
\cmidrule(lr){2-5}
Method & \hdrflat \\
\midrule
\textit{Held-out} & $0.036$ & $0.384$ & $0.588$ & $0.504$ \\
\midrule
RFM           & $0.559 \pm 0.000$ & $1.665 \pm 0.157$ & $0.205 \pm 0.002$ & $0.942 \pm 0.004$ \\
\midrule
GFM-L         & $0.102 \pm 0.013$ & $0.433 \pm 0.004$ & $\underline{0.571} \pm 0.010$ & $0.572 \pm 0.017$ \\
GFM-E         & $0.540 \pm 0.002$ & $0.822 \pm 0.074$ & $0.313 \pm 0.008$ & $0.888 \pm 0.006$ \\
GFM-S         & $0.183 \pm 0.008$ & $0.493 \pm 0.002$ & $0.473 \pm 0.012$ & $0.756 \pm 0.006$ \\
GFM-MF        & $0.558 \pm 0.001$ & $1.195 \pm 0.233$ & $0.248 \pm 0.021$ & $0.926 \pm 0.007$ \\
\midrule
RMF-L$x$      & $0.564 \pm 0.007$ & $2.219 \pm 0.723$ & $0.150 \pm 0.088$ & $0.971 \pm 0.024$ \\
RMF-L$v$      & $0.195 \pm 0.008$ & $0.473 \pm 0.003$ & $0.545 \pm 0.001$ & $0.669 \pm 0.021$ \\
RMF-E$x$      & $0.490 \pm 0.007$ & $0.660 \pm 0.033$ & $0.343 \pm 0.007$ & $0.872 \pm 0.006$ \\
RMF-E$v$      & $0.336 \pm 0.008$ & $0.526 \pm 0.007$ & $0.449 \pm 0.011$ & $0.790 \pm 0.008$ \\
RMF-S$x$      & $0.560 \pm 0.000$ & $1.760 \pm 0.160$ & $0.209 \pm 0.011$ & $0.939 \pm 0.002$ \\
RMF-S$v$      & $0.094 \pm 0.006$ & $0.441 \pm 0.002$ & $0.498 \pm 0.001$ & $0.644 \pm 0.017$ \\
\midrule
KGD-GG        & $0.560 \pm 0.000$ & $1.759 \pm 0.052$ & $0.204 \pm 0.007$ & $0.948 \pm 0.004$ \\
KGD-GL        & $0.559 \pm 0.000$ & $1.580 \pm 0.278$ & $0.213 \pm 0.003$ & $0.940 \pm 0.007$ \\
\midrule
RCM-CT           & $0.416 \pm 0.004$ & $0.522 \pm 0.006$ & $0.407 \pm 0.004$ & $0.809 \pm 0.016$ \\
RCM-sCT          & $0.544 \pm 0.001$ & $0.850 \pm 0.011$ & $0.311 \pm 0.009$ & $0.889 \pm 0.001$ \\
\midrule
MMDF-G        & $0.400 \pm 0.276$ & $1.360 \pm 0.534$ & $0.290 \pm 0.123$ & $0.861 \pm 0.136$ \\
MMDF-L        & $\mathbf{0.043} \pm 0.003$ & $0.449 \pm 0.007$ & $0.565 \pm 0.006$ & $0.535 \pm 0.013$ \\
MMDF-H        & $0.396 \pm 0.283$ & $1.400 \pm 0.632$ & $0.308 \pm 0.167$ & $0.856 \pm 0.146$ \\
\midrule
\ourmethod-SG & $0.105 \pm 0.001$ & $0.437 \pm 0.009$ & $0.551 \pm 0.006$ & $\underline{0.519} \pm 0.010$ \\
\ourmethod-C  & $\underline{0.046} \pm 0.005$ & $\mathbf{0.418} \pm 0.003$ & $0.565 \pm 0.012$ & $0.532 \pm 0.012$ \\
\ourmethod-G  & $0.062 \pm 0.004$ & $\underline{0.428} \pm 0.008$ & $\mathbf{0.580} \pm 0.006$ & $\mathbf{0.508} \pm 0.011$ \\
\ourmethod-H  & $0.104 \pm 0.003$ & $0.438 \pm 0.008$ & $0.552 \pm 0.010$ & $0.520 \pm 0.008$ \\
\bottomrule
\end{tabularx}
\label{tab:torus-results-full-c}
\end{subtable}
\end{table*}

\providecommand{\hdrflatkmmdmmdcovnnatime}{\kmmd$\downarrow$ & \mmd$\downarrow$ & \cov$\uparrow$ & \nna & Time\,(s)$\downarrow$}

\begin{table*}[h]
\centering
\caption[Mesh results under RFM's own training budget (full)]{\textbf{Mesh results under RFM's own training budget (full).} We report \kmmd{}, \mmd{}, \cov{}, \nna{} (closest to 0.5 is best), and sampling time on the Bunny and Spot datasets (\textbf{bold} = best, \underline{underline} = second best; the held-out row is the val-vs-test reference and is excluded from ranking). RFM is trained for its shipped default of 100k iterations and sampled with 1{,}000 steps of its projected Euler integrator (RFM$_{1000}$); \ourmethod-H uses the heat cost, and is sampled with NFE\,$=$\,1. Time is the wall time to draw 2000 samples on one RTX 3090.}
\begin{subtable}{\linewidth}
\centering
\scriptsize
\setlength{\tabcolsep}{2pt}
\begin{tabularx}{\linewidth}{l*{10}{Y}}
\toprule
 & \multicolumn{5}{c}{Bunny ($k=10$)} & \multicolumn{5}{c}{Bunny ($k=50$)} \\
\cmidrule(lr){2-6} \cmidrule(lr){7-11}
Method & \hdrflatkmmdmmdcovnnatime & \hdrflatkmmdmmdcovnnatime \\
\midrule
\textit{Held-out} & $0.020$ & $0.014$ & $0.583$ & $0.493$ & --- & $0.025$ & $0.014$ & $0.583$ & $0.498$ & --- \\
\midrule
RFM$_{1000}$ & $\underline{0.048}$ & $\mathbf{0.015}$ & $\mathbf{0.569}$ & $\mathbf{0.514}$ & $\underline{49.5}$ & $\mathbf{0.028}$ & $\mathbf{0.016}$ & $\mathbf{0.565}$ & $\mathbf{0.516}$ & $\underline{48.7}$ \\
\midrule
\ourmethod-H & $\mathbf{0.030}$ & $\mathbf{0.015}$ & $\underline{0.558}$ & $\underline{0.537}$ & $\mathbf{0.16}$ & $\underline{0.034}$ & $\underline{0.017}$ & $\underline{0.544}$ & $\underline{0.550}$ & $\mathbf{0.13}$ \\
\bottomrule
\end{tabularx}
\caption{Bunny ($k=10$) and Bunny ($k=50$)}
\label{tab:mesh-clock-full-a}
\end{subtable}
\vspace{6pt}
\begin{subtable}{\linewidth}
\centering
\scriptsize
\setlength{\tabcolsep}{2pt}
\begin{tabularx}{\linewidth}{l*{10}{Y}}
\toprule
 & \multicolumn{5}{c}{Spot ($k=10$)} & \multicolumn{5}{c}{Spot ($k=50$)} \\
\cmidrule(lr){2-6} \cmidrule(lr){7-11}
Method & \hdrflatkmmdmmdcovnnatime & \hdrflatkmmdmmdcovnnatime \\
\midrule
\textit{Held-out} & $0.012$ & $0.016$ & $0.587$ & $0.500$ & --- & $0.014$ & $0.016$ & $0.584$ & $0.515$ & --- \\
\midrule
RFM$_{1000}$ & $\mathbf{0.016}$ & $\mathbf{0.017}$ & $\mathbf{0.561}$ & $\mathbf{0.522}$ & $\underline{54.8}$ & $\underline{0.023}$ & $\mathbf{0.017}$ & $\underline{0.540}$ & $\mathbf{0.549}$ & $\underline{54.7}$ \\
\midrule
\ourmethod-H & $\underline{0.018}$ & $\underline{0.021}$ & $\underline{0.493}$ & $\underline{0.601}$ & $\mathbf{0.24}$ & $\mathbf{0.015}$ & $\underline{0.019}$ & $\mathbf{0.549}$ & $\underline{0.559}$ & $\mathbf{0.24}$ \\
\bottomrule
\end{tabularx}
\caption{Spot ($k=10$) and Spot ($k=50$)}
\label{tab:mesh-clock-full-b}
\end{subtable}
\label{tab:mesh-results-full}
\end{table*}

\end{document}